\documentclass[accepted]{uai2026} 
                        
\usepackage[american]{babel}

\usepackage{natbib} 
\usepackage{mathtools} 
\usepackage{booktabs} 
\usepackage{tikz} 

\usepackage{amsmath}
\usepackage{algorithm}
\usepackage{algorithmic}
\usepackage{amssymb}
\usepackage{mathtools}
\usepackage{amsthm}
\usetikzlibrary{positioning}
\usepackage{comment}

\theoremstyle{plain}
\newtheorem{theorem}{Theorem}[section]
\newtheorem{proposition}[theorem]{Proposition}
\newtheorem{lemma}[theorem]{Lemma}

\theoremstyle{definition}
\newtheorem{definition}[theorem]{Definition}
\newtheorem{assumption}{Assumption}
\theoremstyle{remark}
\newtheorem{remark}[theorem]{Remark}
\theoremstyle{exmp}
\newtheorem{exmp}{Example}

\title{Probably Correct Optimal Stable Matching\\ under Two-Sided Uncertainty}

\author[1]{Andreas Athanasopoulos}
\author[2]{Anne-Marie George}
\author[1]{Christos Dimitrakakis}
\affil[1]{%
    University of Neuchâtel\\
    Neuchâtel, Switzerland
}
\affil[2]{%
    University of Oslo\\
    Oslo, Norway
}

\newcommand{\util}{\mu}

\newcommand{\agent}{\mathrm{a}}
\newcommand{\setAgents}{\mathfrak{A}} 

\newcommand{\setPlayer}{\mathcal{P}} 
\newcommand{\player}{p}              

\newcommand{\setArms}{\mathcal{A}}
\newcommand{\arm}{a}              

\newcommand{\expr}{\mu} 

\newcommand{\matching}{\mathfrak{m}}            
\newcommand{\Matching}{\mathbb{M}}

\DeclareMathOperator*{\argsort}{arg\,sort}

\begin{document}
\maketitle

\begin{abstract}
  We study a sequential learning problem for stable matchings in two-sided markets where preferences on both sides are initially unknown.  We focus on a centralized setting where an algorithm matches agents at each time step and receives noisy rewards that reflect the preferences of the matched agents, following a semi-bandit feedback structure. We adopt a pure exploration perspective, aiming to efficiently identify the optimal stable matching with high probability. Our work extends prior results by handling \emph{two-sided uncertainty} and by exploiting \emph{partial preference} information. A central ingredient is the notion of \textbf{pervasive stable matching}, which enables the identification of optimal stable matchings under partial preferences. We propose elimination-based algorithms whose stopping criteria exploit the structure of the learned partial preferences, and provide a refined sample-complexity analysis. Beyond pure exploration, we extend our approach to regret minimization and establish regret bounds with respect to the \emph{optimal} stable matching that avoid dependence on the minimum reward gap $\Delta_{\min}$.
\end{abstract}

\section{Introduction}
\label{sec:intro}
Two-sided matching markets model situations in which two distinct sets of agents aim to match with each other based on their preferences~\citep{gusfield1989stable}. Classic examples include students applying to universities, workers matching with firms, and organ donors paired with patients~\citep{roth1984evolution}. A fundamental solution concept in this setting is \emph{stable matching}~\citep{gale_Shapley_1962}, a notion of equilibrium in the market, where no pair of agents has an incentive to deviate from the proposed assignment.

Classical models typically assume that preferences are known in advance. In contrast, modern applications such as ride-sharing systems and online labor markets motivate sequential learning formulations, where agents’ preferences are initially unknown and must be learned over time through repeated interactions.

A growing body of literature casts this learning scenario within the Multi-Armed Bandit (MAB) framework, a direction initiated by \citet{das2005two}. In the centralized version of the problem, a platform matches the agents and receives noisy feedback reflecting their underlying preferences. This setting creates an inherent tension between exploration and exploitation: the platform must decide whether to exploit current estimates of a stable matching or continue exploring to learn the agents' preferences.

The problem is typically studied within a regret-minimization framework. In this work, we instead adopt a pure exploration perspective, aiming to efficiently identify an \emph{optimal stable matching} with high probability.

\subsection{Related work}
In their foundational work, \citeauthor{das2005two} introduce this learning problem and empirically evaluate the performance of MAB algorithms across several preference structures. A follow-up by \citet{pmlr-v108-liu20c} introduced the notions of \emph{Player-Optimal} and \emph{Player-Pessimal} stable regret. Focusing on a setting where only one side of the market has unknown preferences, they analyzed an Explore-Then-Commit algorithm, proving sublinear bounds for both regret notions. They also examined an Upper Confidence Bound approach, establishing sublinear bounds for Player-Pessimal stable regret; however, they proved an impossibility result for bounding Player-Optimal regret.

Later studies closed this gap by achieving sublinear player-optimal stable regret \citep{sankararaman2021dominate, basu2021beyond, kong2023player}, though they primarily focused on one-sided uncertainty and specific market structures. More recent studies have begun to consider settings where both sides of the market have uncertain preferences \citep{pagare2024explore, pmlr-v244-zhang24b}. However, these works employ Explore-Then-Commit type algorithms. In contrast, this work adopts a pure exploration framework that is more naturally suited to the analysis of such algorithms and enables refined results on regret with respect to the \emph{optimal} stable matching.

The pure exploration objective of identifying an \emph{optimal stable matching} with high probability was introduced by \cite{athanasopoulos2025probably}, while the work of \cite{NEURIPS2024_7fc0aac0} focuses on identifying \emph{any} stable matching. Both of these studies propose elimination-based algorithms restricted to one-sided uncertainty. We extend this literature by tackling a two-sided learning environment, providing refined arguments for the analysis of elimination algorithms. 

A core component of our approach is utilizing the agents' \emph{partial preference} rankings to identify an optimal stable matching, thereby circumventing the need to learn their full preference orderings. This approach is motivated by the concept of \emph{pervasive stable matching}, introduced by \cite{rastegari2014reasoning} to reason about optimal stable matchings under partial information in \emph{non-learning} scenarios. Specifically, a partial preference profile of the agents admits a pervasive stable matching if that matching remains optimally stable under every possible full refinement of the given preferences. A closely related notion is that of a \emph{super stable matching}, initially introduced to address stability in instances with indifferent preferences \citep{IRVING1994261}. While super stability was recently leveraged by \cite{basu2025competing} to bound regret in sequential learning scenarios, our work specifically relies on pervasive stable matchings. This distinction is crucial, as our objective is to identify the \emph{optimal} stable matching, rather than \emph{any} stable matching.

Appendix~\ref{app:related_work} contains an extended review of related literature, along with a comprehensive discussion on metrics.

\subsection{Contributions}

In this work, we demonstrate how partial preferences and the concepts of \emph{pervasive stable matching} can be applied in a learning scenario to identify the optimal stable matching with high probability under \emph{two-sided market uncertainty}. 

Our specific contributions are structured as follows:

\begin{itemize}
    \item \textbf{Uniform Strategies:} We present two variants of uniform exploration strategies that assume additional information about the reward structure. We show that reasoning about partial preferences significantly reduces the overall sample complexity compared to approaches that learn the full rankings, thereby avoiding the traditional dependence on the minimum reward gap.
    \item \textbf{Elimination Algorithm \& Technical Contribution:} We propose an adaptive elimination algorithm that requires no prior knowledge of reward gaps, utilizing a stopping rule based on partial preferences. Furthermore, we provide a refined complexity bound based on edge-coloring arguments on bipartite graphs.
    \item \textbf{Extended Elimination Rule:} We also present an extension of the elimination algorithm that further eliminates pairs upon identifying a super stable matching.
    \item \textbf{Regret Minimization:} Beyond our primary objective of pure exploration, we propose a meta-algorithm for the regret minimization setting. To the best of our knowledge, this provides the first regret bound with respect to the \emph{optimal} stable matching under \emph{two-sided uncertainty} that avoids a dependence on the minimum reward gap $\Delta_{min}$.
    \item \textbf{Theoretical Lower Bounds:} Finally, we conclude our analysis by establishing a lower bound on the sample complexity, utilizing known results from the combinatorial MAB literature \citep{jourdan2021efficient}.
\end{itemize}

The remainder of the paper is organized as follows. Section~\ref{sec:stable-matching} introduces the necessary notation and background regarding the stable matching problem and partial preferences. Section~\ref{sec:Setting} then formally defines our learning problem and explains how we utilize the concept of pervasive stable matching. In Section~\ref{sec:etc}, we analyze uniform sampling strategies, discussing the impact of utilizing partial preferences. Section~\ref{sec:elimination_algo} presents our elimination-based algorithm, followed by an extension of the elimination procedure and a discussion on a meta-algorithm for the regret minimization setting. Next, Section~\ref{sec:lower_bound}  establishes a lower bound on the sample complexity. In Section~\ref{sec:sim}, we evaluate our algorithms through simulation on random instances. Further discussion and omitted proofs are provided in the Appendix.

\section{Background}
\label{sec:stable-matching}
We consider a learning problem in a two-sided matching market, consisting of two distinct, non-empty sets of agents, $\setPlayer$ and $\setArms$, with $K = |\setPlayer|$ and $N = |\setArms|$ denoting the number of agents on each side, with arbitrary $K, N$\footnote{Prior work usually assume $K \leq N$ \citep{liu2021bandit,basu2025competing,kong2023player,athanasopoulos2025probably}.}.

The agents $\setAgents = \setPlayer \cup \setArms$ are matched one-to-one according to a matching $\Matching \subseteq \setPlayer \times \setArms$, i.e., pairs in $\Matching$ are pairwise disjoint. For convenience, we use the functional representation of a matching, where $\matching(\player) = \arm$ and $\matching(\arm) = \player$ for every $(\player, \arm) \in \Matching$, and $\matching(\agent) = \bot$ for any unmatched agent $\agent \in \setAgents$. We  define \( OS(\agent) \) as the opposite side of agent \( \agent \), i.e., \( OS(\agent) = \setArms\) if \( \agent \in \setPlayer \), and \( OS(\agent) = \setPlayer \) if \( \agent \in \setArms \).

Each agent $\agent \in \setAgents$ has an underlying true preference over the agents on the opposite side, represented by a strict preference ranking $F_{\agent}$. These preferences are derived from utility functions $\util_{\agent} : OS(\agent) \cup \{\bot\} \rightarrow \mathbb{R}$, where $\util_{\agent, \matching(\agent)}$ denotes the utility that agent $\agent$ obtains when matched with $\matching(\agent) \in OS(\agent)$, and $\util_{\agent, \bot}$ is the utility of being unmatched. We assume that for every agent $\agent \in \setAgents$, the utility values $\{\util_{\agent, i}\}_{i \in OS(\agent) \cup \{\bot\}}$ are distinct, thereby ensuring no ties in preferences. The preference ranking is then uniquely constructed by sorting these utilities: $F_{\agent} = \argsort_{i \in OS(\agent) \cup \{\bot\}} \util_{\agent, i}$. We denote by $\succ_{F_{\agent}}$ the corresponding strict order relation over $OS(\agent) \cup \{\bot\}$, where $i \succ_{F_{\agent}} j$ if and only if $\util_{\agent, i} > \util_{\agent, j}$ for $i, j \in OS(\agent) \cup \{\bot\}$

In order for a matching $\matching$ to align with the agents’ preferences, \citet{gale_Shapley_1962} introduced the concept of stable matching, a notion of equilibrium in which no two agents have an incentive to deviate. Formally:
\begin{definition}[Stable Matching]
A matching $\matching$ is stable under the true preference profile $\{F_{\agent}\}_{\agent \in \setAgents}$ if:\begin{enumerate}
    \item \textbf{Individual rationality:} For every agent $\agent \in \setAgents$, it holds that $\matching(\agent) \succeq_{F_{\agent}} \bot$.
    \item \textbf{No blocking pairs:} There exists no pair $(p, a) \in \setPlayer \times \setArms$ such that agent $p$ prefers $a$ over $\matching(p)$ and agent $a$ prefers $p$ over $\matching(a)$, i.e., $a \succ_{F_p} \matching(p)$ and $p \succ_{F_a} \matching(a)$.
\end{enumerate}
\end{definition}

\citet{gale_Shapley_1962} proved that the set of stable matchings $\mathcal{S}$ is always non-empty by introducing the Gale–Shapley (GS) algorithm, which is based on sequential proposals from one side of the market to the other. The algorithm returns the (unique) \emph{optimal stable matching} $m^\star$ for the proposing side, in the sense that every proposer receives the best partner they can obtain in any stable matching. Conversely, the resulting matching is pessimal for the side receiving the proposals, providing each agent on that side with their least-preferred partner among all stable matchings. In the remainder of the paper we concentrate on $\setPlayer$-optimal stable matchings $m^\star$, simply referred to as \textit{optimal}.

\subsection{Stable Matching from Partial Preferences}
\label{sec:partial-matching}
If we correctly estimate the full preference rankings of the agents, we can identify the optimal stable matching using the GS algorithm. However, learning the complete preference orders is not necessary: the optimal stable matching can be identified using only a correct partial preference order, as described by~\cite{rastegari2014reasoning}.

Stable matchings under certain types of partial preferences form the basis of our proposed algorithms. We begin with the definition of a partial ranking.

\begin{definition}[Partial ranking]
A partial ranking $P_a$ of an agent $a\in\setAgents$ over the set of agents $OS(a)$ is a partial order $\succ_{P_a}$, i.e., asymmetric and transitive, that can be represented as directed acyclic graph (DAG) with edges $(i, j) \subseteq OS(a) \times OS(a)$ for all $i\neq j$ with $i \succ_{P_a} j$ (meaning $i \not\preceq_{P_a} j$).
\end{definition}
Note that all partial rankings are asymmetric and transitive, and that a full ranking is a special case of a partial ranking.
\begin{definition}[Full ranking]
A full ranking \(F_a\) of an agent \(a\) over the set of agents \(OS(a)\) is a partial ranking 
such that, for every distinct pair \((i, j) \in OS(a) \times OS(a)\), either 
\(i \succ_{F_a} j\) or \(j \succ_{F_a} i\).
\end{definition}
We now define when a partial ranking \(P_a\) is compatible with \(P'_a\), in the sense that \(P'_a\) is a refinement of \(P_a\).
\begin{definition}[Compatible Partial Rankings]
A partial ranking \(P_a\) is said to be \emph{compatible} with another partial
ranking \(P'_a\), denoted by \(P_a \mapsto P'_a\), if for any \(i, j \in OS(a)\) such that 
\(i \succ_{P'_a} j\), it is not the case that \(j \succ_{P_a} i\). We overload this definition to partial \textit{preference profiles} \(P_{A} = \{P_a\}_{a \in \setAgents}\) and 
\(P'_{A} = \{P'_a\}_{a \in \setAgents}\) of the agents,
writing \(P_{A} \mapsto P'_{A}\) if \(P_a \mapsto P'_a\) for all 
\(a \in \setAgents\).
\end{definition}

Given a partial preference profile \(P_{A}\), we can describe the set of full ranking profiles \(F_{A} = \{F_a\}_{a \in \setAgents}\) that refine \(P_{A}\) as \(\mathcal{F}(P_{A}) = \{F_{A} \mid P_{A} \mapsto F_{A}\}\). This motivates the following definition of a super stable matching \citep{IRVING1994261}, a concept originally proposed for preferences with indifferences: we say that a matching $\matching$ is \emph{super stable} w.r.t. the preference profile $P_A$ if it is stable w.r.t. every full ranking profile $F_A \in \mathcal{F}(P_A)$.

The work of \cite{rastegari2014reasoning} introduce a similar definition of \emph{pervasive stable matching} (PSM), for a matching that is stable and optimal with respect to every full refinement of the partial preferences $F_A \in \mathcal{F}(P_A)$.

\begin{definition}[Pervasive Stable Matching]
We say that a matching $\matching$ is a \emph{pervasive stable matching} 
for the preference profile $P_A$ if $\matching$ is the optimal stable matching 
for every full ranking profile $F_A \in \mathcal{F}(P_A)$. 
We also say that $P_A$ \emph{admits a pervasive stable matching} 
if such a matching exists.
\end{definition}

Previous work shows that the set of super stable matchings may be empty, and whenever it is non-empty, it forms a distributive lattice \citep{spieker1995set}. In this case, there exists a unique optimal super stable matching for each side. In the following example we highlight that an optimal super stable matching is not necessarily the same as the pervasive stable matching \citep{rastegari2014reasoning}.

\begin{exmp}\label{ex:SS-not-POSM}
Consider an instance with two agents on each side, $\setPlayer = \{p_1, p_2\}$ and $\setArms=\{a_1, a_2\}$, respectively. We assume that the agents in $\setArms$ have correctly estimated their preferences, while agents in $\setPlayer$ have no knowledge about their preferences yet, i.e. $P_p =\varnothing \; \forall  p \in \setPlayer$, making them indifferent between $ \arm_1 $ and $ \arm_2$, denoted by $ \arm_1 \sim \arm_2$:
\begin{center}
\begin{tabular}{l l l l}
\toprule
\toprule
 \quad $p_1$ &  \quad $p_2$ &  \quad $a_1$ &  \quad $a_2$ \\
\midrule
 $\arm_1 \sim \arm_2$ &$\arm_1 \sim \arm_2$ &   $\player_1 \succ \player_2$ &  $\player_2 \succ \player_1$ \\[2pt]
\bottomrule
\bottomrule
\end{tabular}
\end{center}
In this setting, there are four possible full ranking profiles in $\mathcal{F}(\{P_a\}_{a \in \setAgents})$ that can refine the partial preferences of the agents in $\setPlayer$ that we illustrate in the next table.
\begin{center}
\begin{tabular}{c | l l | l l}
\toprule
\toprule
Instance & \quad $p_1$ &  \quad $p_2$ &  \quad $a_1$ &  \quad $a_2$ \\
\midrule
$I_1$ & $\arm_1 \succ \arm_2$ &$\arm_1 \succ \arm_2$ &    $\player_1 \succ \player_2$ &    $\player_2 \succ \player_1$ \\[2pt]
$I_2$ & $\arm_1 \succ \arm_2$ &$\arm_2 \succ \arm_1$ &    $\player_1 \succ \player_2$ &   $\player_2 \succ \player_1$ \\[2pt]
$I_3$ & $\arm_2 \succ \arm_1$ &$\arm_1 \succ \arm_2$ &  $\player_1 \succ \player_2$ &  $\player_2 \succ \player_1$ \\[2pt]
$I_4$ & $\arm_2 \succ \arm_1$ &$\arm_2 \succ \arm_1$ &  $\player_1 \succ \player_2$ &  $\player_2 \succ \player_1$ \\[2pt]
\bottomrule
\bottomrule
\end{tabular}
\end{center}

The only super stable matching, and thus the optimal, is $M= \{(p_1,a_1), (p_2, a_2)\}$ as the matching is stable with respect to every possible profile $I_j$ with $j = 1, \dots, 4$. 

Note that $M$, is not a pervasive stable matching, as for $I_3$ the optimal stable matching is $ \{(p_1,a_2), (p_2, a_1)\}$.

If $p_1$ collects more samples to decide that $a_1 \succ a_2$, then $M$ is PSM, as it is the optimal stable for both $I_1$ and $I_2$.
\end{exmp}

\section{Learning Problem}
\label{sec:Setting}
We consider a learning problem where agents are initially unaware of their preferences and must learn them through repeated interactions. We focus on the centralized setting, where an algorithm selects a matching at each time step and observes stochastic rewards for the matched pairs~\footnote{Equivalently, we can assume that agents receive the rewards and truthfully report them to the algorithm.}. 

More formally, the algorithm interacts with an instance
$\nu = \{ \nu_{a,i} \}_{(a, i) \in \setPlayer\times\setArms} \in \mathcal{I}$, where each $\nu_{a,i}$ corresponds to the reward distribution of agent $a \in \setAgents$ when matched
with agent $i \in OS(a)$, and has mean $\mu_{a,i} = \mathbb{E}_{X \sim \nu_{a,i}}[X].$
\begin{assumption}[$1$-sub-Gaussian]
For every agent $\agent \in \mathfrak{A}$ and $i \in OS(\agent)$, the reward distribution $\nu_{\agent, i}$ is $1$-sub-Gaussian.
\end{assumption}
At each time step $t \ge 1$, the algorithm selects a matching $m_t \in \mathcal{M}$ out of the set of all matchings $\mathcal{M}$, and observes stochastic rewards $X_{a,m_t(a)} \sim \nu_{a,m_t(a)}$ for each matched pair $(a, m_t(a))$ for all $a \in \mathfrak{A}$. Specifically, the feedback at time $t$ is represented by the vector $X_t = (X_{a,m_t(a)} \mathbb{I}\{a \in m_t(a)\})_{a \in \mathfrak{A}} \in \mathbb{R}^{K \times N}$. The filtration $\mathcal{F}_t := \sigma(m_1, X_1, \dots, m_t, X_t)$ contains all the information available at step $t+1$, where $\sigma(\cdot)$ denotes the $\sigma$-algebra generated by its arguments.

The algorithm is defined by:
\begin{enumerate}
\item A \textbf{Sampling Rule}: A rule determining which matching to 
sample at each step $t$, formally a sequence $(m_t)_{t \in 
\mathbb{N}}$ where each $m_t$ is $\mathcal{F}_{t-1}$-measurable.
\item A \textbf{Stopping Rule}: A stopping time $\tau$ with respect to the filtration $(\mathcal{F}_t)_{t \in \mathbb{N}}$, which determines when the exploration phase ends.
\item A \textbf{Recommendation Rule}: An $\mathcal{F}_{\tau}$-measurable rule that selects a final matching $m_\tau$ to be returned upon termination.
\end{enumerate}

We consider the fixed-confidence setting, where the objective is to design an algorithm that terminates and outputs the correct optimal stable matching with high probability. 
\begin{definition}[Probably Correct Optimal Stable-matching]
    Given a fixed confidence parameter $\delta \in (0, 1)$, an algorithm is said to be $\delta$-PCOS for every $\nu \in \mathcal{E}_{\Matching}$ if it terminates after $\tau < \infty$ steps and returns the optimal stable matching $m_{\nu}^{\star}$ with high probability:
    \begin{equation*}
        \mathbb{P}_{\nu}(m_\tau = m_{\nu}^{\star}) \geq 1 - \delta.
    \end{equation*}
\end{definition}

\begin{remark}
Beyond the probability of correctness, we are also interested in providing guarantees for the stopping time $\tau$, which represents the total number of matchings sampled before termination. We refer to the expected number of steps $\mathbb{E}[\tau]$ as the sample complexity of the algorithm.
\end{remark}

\paragraph{Learning with Partial Preferences.} Our previous discussion in Section~\ref{sec:partial-matching} indicates that an algorithm only needs to learn a valid partial preference profile $P_A$ that admits a pervasive stable matching. This insight motivates the design of our algorithms in the subsequent sections. 

\begin{definition}[Valid Preferences]
We say that a partial preference profile $P_A$ is \emph{valid} for the 
full ranking profile $F_{A}$ of an instance $\nu$ if:
\begin{enumerate}
    \item $P_A$ is compatible with $F_{A}$, i.e., $P_A \mapsto F_A$, and
    \item $P_A$ admits a pervasive stable matching.
\end{enumerate}
\end{definition}

We denote the set of valid preference profiles by $\mathcal{P}_{\text{valid}}(F_{A}^{\nu})$, which is never empty as it always contains the true profile $F_{A}$. While the work of \cite{rastegari2013two} proves that identifying the set of valid preference profiles is NP-hard, their follow-up research provides a polynomial-time algorithm to return the pervasive stable matching whenever it exists, given a partial ranking \citep{rastegari2014reasoning}.

To analyze the sample complexity of our algorithms, we introduce an admissible gap $\Delta_{F}$ associated with an instance and the set of valid partial preference profiles $\mathcal{P}_{\mathrm{valid}}(F_A^{\nu})$. 

We define the minimum gap induced by a partial ranking $P_a$ as $\Delta_{\min}(P_a) = 
\min_{(i,j) \in P_a}
\lvert \mu_{a,i} - \mu_{a,j} \rvert .$
Thus $\Delta_{\min}(P_a)$ indicates how hard it is for agent $a$ to learn their preference between pair $(i,j)\in P_a$ in the worst case.
For a preference profile $P_A = \{P_a\}_{a \in \setAgents}$, we extend this to $
\Delta_{\min}(P_A) = \min_{a \in \setAgents} \Delta_{\min}(P_a)$. Finally, the \textit{admissible gap of an instance} is defined as the maximum such minimum gaps over all valid profiles, i.e., 
\begin{equation}
        \Delta_{F} = \max_{P_A \in \mathcal{P}_{\mathrm{valid}}(F_A^{\nu})} \Delta_{\min}(P_A).
\end{equation}

\textbf{Relation to the minimum gap:} Previous work typically characterizes the performance of an algorithm in terms of the minimum reward gap $\Delta_{\min} = \min_{\agent \in \mathfrak{A}} \Delta_{\min}(F_a)$ \citep{pmlr-v108-liu20c}. In contrast, our admissible gap $\Delta_F$, defined with respect to the minimum valid partial profile, satisfies $\Delta_F \geq \Delta_{\min}$ by definition.

\section{Uniform strategies}
\label{sec:etc}
Our first contribution using partial preferences is the analysis of a uniform strategy that samples matchings covering all possible pairs $\mathcal{P} \times \mathcal{A}$ uniformly for a fixed number of $h$ exploration rounds. We analyze this algorithm for two cases, where the number of exploration rounds $h$ uses $\Delta_{min}$ or $\Delta_F$, to ensure the identification of the correct optimal stable matching with high probability.

After uniformly sampling pairs in $\mathcal{P} \times \mathcal{A}$ for $h$ rounds, Algorithm~\ref{alg:etc} commits to a matching by executing GS on the estimated preferences, $m_{\tau} = \text{GS}(\{\hat{F}_\agent\}_{\agent \in \setAgents})$. Here the preference ranking for an agent $\agent \in \setAgents$ is determined by sorting the empirical mean rewards, i.e., $\hat{F}_{\agent} = \argsort_{i \in OS(\agent)} \hat{\mu}_{\agent,i}(t)$, where $\hat{\mu}_{\agent,i}(t) = \frac{1}{N_{\agent,i}(t)} \sum_{s=1}^{t} \mathbb{I}(m_s(\agent) = i ) X_{\agent,i}(s)$ and $N_{\agent,i}(t)$ denotes the number of times the pair $(\agent,i)$ has been sampled up to time $t$.

\paragraph{Sampling from Minimal Matching Covers:} To sample all pairs in a set $X \subseteq \mathcal{P} \times \mathcal{A}$, we sample matchings from a minimal matching cover $\mathfrak{M}(X)$ --- the smallest set of matchings such that $X \subseteq \bigcup_{m \in \mathfrak{M}(X)} m$. For a bipartite graph, finding this cover is equivalent to a minimum edge coloring problem, where the edges of each color correspond to a matching in the cover \citep{gabow1978algorithms}. According to Kőnig's Theorem, the minimal number of matchings in a cover is equal to the maximum degree $deg(G(X))$ of the bipartite graph $G(X)$ with edges $X$. Thus, in our setting, a full exploration of all pairs $\mathcal{P} \times \mathcal{A}$ requires $\max\{K, N\}$ matchings in each exploration round $h$. Sampling from a matching cover  is essential for our later elimination algorithms, where only a specific subset of pairs needs to be sampled in each step. It relaxes the requirement for round-robin strategies and extends the analysis beyond the usual assumption of $N < K$. 

\begin{algorithm}[t]
\caption{Uniform Exploration}\label{alg:etc}
\begin{algorithmic}[1]
\REQUIRE h, $\setPlayer$, $\setArms$, $n=\max\{|\setPlayer|, |\setArms|\}$
\STATE Compute a minimal matching cover\\ $\{\matching_i\}^{n}_{i=1} = \mathfrak{M}(\{(p,a) \in \setPlayer \times \setArms\})$
\FOR{ $t$ \ in \ $\{1, \cdots, h\cdot n\}$}
    \STATE  \textit{Sample } $m_{i+ 1}$ \textit{where} $i = (t-1) \mod n$.
    \STATE \textit{Update $\hat{\expr}_{a,\matching_{i+1}(a)}(t)$  $\forall a \; \in \; \setAgents $
    }
\ENDFOR
\STATE $\hat{F}_{\agent} = \argsort_{i \in OS(\agent)} \hat{\expr}_{\agent, i} \; \forall \agent \in \setAgents$ 
\STATE \textbf{return} $m_{\tau} = GS(\{\hat{F}_\agent\}_{\agent \in\setAgents})$ 
\end{algorithmic}
\end{algorithm}

We can now bound the sample complexity of the uniform strategy where the number of exploration rounds is defined as a function of $\Delta_{\min}$. This algorithm samples pairs until all agents \textit{complete} preference rankings can be determined.

\begin{theorem}\label{th:theorem_etc1}Algorithm~\ref{alg:etc} with parameter $ h =  \lceil\frac{8 \ln(2KN/\delta)} {\Delta_{\min}^2}\rceil$ is a $\delta$-PCOS algorithm, and the number of sampled matchings is bounded by:
\begin{equation*}
    O\left(\max\{K, N\}\frac{\ln(KN/\delta)}{\Delta_{\min}^2}\right).
\end{equation*}
\end{theorem}

Next, we introduce a variant where $h$ depends on $\Delta_F$. In this version, the algorithm samples each pair often enough to correctly identify a valid preference profile $P_A \in \mathcal{P}_{\text{valid}}(F_{A}^{\nu})$, with $\Delta_F= \max_{P_A \in \mathcal{P}_{\mathrm{valid}}(F_A^{\nu})} \Delta_{\min}(P_A)$.

\begin{theorem}\label{th:theorem_etc2}
Algorithm~\ref{alg:etc} with parameter $h =  \lceil\frac{8 \ln(2KN/\delta)} {\Delta_{F}^2}\rceil$ is a $\delta$-PCOS algorithm, and the number of sampled matchings is bounded by:
\begin{equation*}
    O\left(\max\{K, N\}\frac{\ln(KN/\delta)}{\Delta_{F}^2}\right).
\end{equation*}
\end{theorem}

\begin{remark}
Note that by construction $\Delta_{\min} \leq \Delta_{F}$, which can lead to an arbitrarily large increase in sample complexity, as illustrated in Example~\ref{exmp_2} and Figure~\ref{fig:sample-complexity-delta}. While  knowing either gap, $\Delta_{\min}$ or $\Delta_{F}$, might not be possible in practice, knowing $\Delta_{F}$ seems to be an even more unrealistic assumption, as computing the set of valid preference profiles is computationally hard \citep{rakhlin2013optimization}. 
\end{remark}

\begin{exmp}
\label{exmp_2}
Consider an instance with two agents on each side, with the following true preference rankings:
\begin{center}
\begin{tabular}{l l l l}
\toprule
 \quad $p_1$ & \quad $p_2$ & \quad $a_1$ & \quad $a_2$ \\
\midrule
 $a_1 \succ a_2$ & $a_1 \succ a_2$ &
 $p_1 \succ p_2$ & $p_1 \succ p_2$ \\
\bottomrule
\end{tabular}
\end{center}

For this instance, the partial preference profile $P_A=\{P_{p_1}=\{a_1 \succ a_2\},P_{p_2}=\emptyset,P_{a_1}=\{p_1 \succ p_2\}, P_{a_2}=\emptyset\}$ is valid, i.e., the algorithm does not need to estimate the preferences of $p_2$ and $a_2$. So the difficulty of the learning task does not depend on the gaps associated with these agents.

To illustrate, consider the following mean rewards:
\begin{center}
\begin{tabular}{c c c}
\toprule
Agent & \multicolumn{2}{c}{Mean rewards $\mu$} \\
\midrule
$p_1$ & $\mu_{p_1,a_1} = 1$ & $\mu_{p_1,a_2} = 1-\Delta_F$ \\
$a_1$ & $\mu_{a_1,p_1} = 1$ & $\mu_{a_1,p_2} = 1-\Delta_F$ \\
$p_2$ & $\mu_{p_2,a_1} = 1$ & $\mu_{p_2,a_2} = 1-\Delta_{\min}$ \\
$a_2$ & $\mu_{a_2,p_1} = 1$ & $\mu_{a_2,p_2} = 1-\Delta_{\min}$ \\
\bottomrule
\end{tabular}
\end{center}

\begin{figure}[t]
    \centering
    \includegraphics[width=0.8\linewidth]{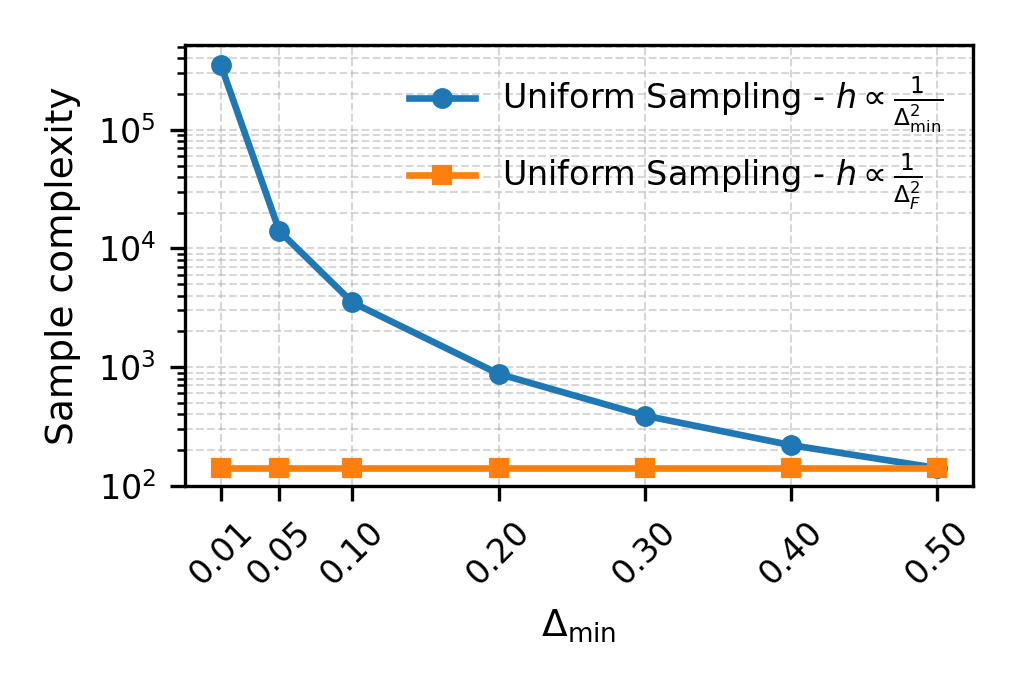}
    \caption{Sample complexity bounds from Theorem~\ref{th:theorem_etc1} and \ref{th:theorem_etc2} for Example~\ref{exmp_2}'s instance with fixed $\Delta_{\min} < \Delta_F = 0.5$.}
    \label{fig:sample-complexity-delta}
\end{figure}

Figure~\ref{fig:sample-complexity-delta} fixes $\Delta_F = 0.5$ and varies $\Delta_{\min}$, comparing the sample complexity of the algorithms with exploration time $h$ defined in Theorem~1 and Theorem~2, and shows that the resulting differences in the sample complexities can become arbitrarily large as $\Delta_{\min}$ goes to zero.
\end{exmp}

\section{Elimination Algorithm}
\label{sec:elimination_algo}
We now turn our attention to elimination algorithms: these have improved sample complexity and do not require prior knowledge of gaps $\Delta$.

\subsection{Simple Elimination Algorithm}
\label{sec:elim_psm}
The key idea of Algorithm~\ref{alg:elim} is to (i) update partial preferences once pairwise orderings are inferred with high probability, (ii) eliminate pairs from sampling, and (iii) stop when a PSM is found. High-probability correctness of the partial preferences implies identification of the optimal matching $m^{\star}$ with high probability.

\paragraph{Sampling:} The elimination algorithm operates in rounds $t$ and keeps track of active (non-eliminated) agents $\mathcal{S}_a(t) \subseteq OS(a)$ -- those for which the relative ordering in $a$'s preferences is not determined with high probability yet. We initialise $\mathcal{S}_a(0) = OS(a)$ for all agents $a\in \setAgents$. Then in each round $t$, we sample the matchings from a \textit{minimal matching cover} $\mathfrak{M}(S_t)$, where $S_t$ is the set of all non-eliminated pairs:
$$S_t = \left\{ (p, a) \in \setPlayer \times \setArms \mid a \in S_p(t) \lor p \in S_a(t) \right\}$$
So at each round $t$, we sample $n_t$ matchings, with $n_t$ defined as the maximum degree of the underling graph of the non-eliminated pairs $G(S_t)$, i.e., $n_t = \text{deg}(G(S_t))$.

\paragraph{Partial Preferences:}
To eliminate pairs and form the partial preferences of the agents, we employ confidence intervals. Specifically, at each round $t$, every agent $a \in \setAgents$ maintains a confidence interval for the mean reward associated with each agent $i \in OS(a)$. The corresponding confidence radius is defined as:
\begin{equation}
\label{eq:ci}
    c_{a,i}(t) = \sqrt{\frac{2\ln{(8KNt^2/\delta)}}{t}}.
\end{equation}
Using this radius, we define the Upper Confidence Bound (UCB) and Lower Confidence Bound (LCB) as:
\begin{align}
    \mathrm{UCB}_{a,i}(t) &= \hat{\mu}_{a,i}(t) + c_{a,i}(t), \\
    \mathrm{LCB}_{a,i}(t) &= \hat{\mu}_{a,i}(t) - c_{a,i}(t).
\end{align}

We construct the partial ranking $P_a$ by adding a directed edge from $i$ to $j$ for any distinct pair $i, j \in OS(a)$ whenever the confidence intervals are disjoint. Formally, at round $t$, we update the edge set of $P_a$ according to:
\begin{equation}
\label{eq:update_pref}
    P_a \leftarrow P_a \cup \{(i, j)\} \quad \text{if} \quad \mathrm{LCB}_{a,i}(t) > \mathrm{UCB}_{a,j}(t).
\end{equation}

\paragraph{Elimination Rule:}
We remove an agent $i$ from a set of active agents $S_a$ if $i$'s confidence interval is strictly disjoint from that of every other active agent $j \in S_a \setminus \{i\}$. Formally, we update $S_a \leftarrow S_a \setminus \{i\}$ if $\forall j \in S_a \setminus \{i\}$:
\begin{equation}
\label{eq:elim}
     \mathrm{LCB}_{a,i}(t) > \mathrm{UCB}_{a,j}(t) \; \lor \; \mathrm{UCB}_{a,i}(t) < \mathrm{LCB}_{a,j}(t).
\end{equation}
This implies that the relative order between $i$ and every other active agent $j$ has been established with high probability, rendering further sampling of $i$ unnecessary for agent $a$.

\paragraph{Stopping Condition:}
We stop sampling either when all available pairs have been eliminated or when the agents’ partial preference profile admits a PSM. By construction, the partial preferences are consistent with the true underlying preferences with high probability; thus, the matching is the optimal stable matching with high probability. To compute a PSM, we use the POSM algorithm \citep{rastegari2014reasoning} (line~3 of Algorithm~\ref{alg:elim}) , which returns a PSM for the given partial preference profile if one exists, and \textsc{None} otherwise, in polynomial time.

\begin{algorithm}[t]
\caption{E-PSM: Elimination PSM}\label{alg:elim}
\begin{algorithmic}[1]
\REQUIRE $\delta$, $\setPlayer$, $\setArms$
\STATE $S_a \leftarrow OS(a) \quad \forall a \in \setAgents$ \hfill $\triangleright$ Active pairs
\STATE $P_a \leftarrow \emptyset \quad \forall a \in \setAgents$ \hfill $\triangleright$ Partial ranking
\STATE $  m_{\tau}= POSM(\{P_a\}_{a \in \setAgents})$
\WHILE{$m_{\tau} = \text{NONE} \land \bigcup_{a \in \setAgents} S_a \neq \emptyset$}
    \STATE $S_t \leftarrow \left\{ (p, a) \in \setPlayer \times \setArms \mid a \in S_p(t) \lor p \in S_a(t) \right\}$
    \FOR{$m \in \mathfrak{M}(S_t)$} 
        \STATE Sample update $\hat{\mu}_{a,m(a)} \; \forall a \in \setAgents$ 
    \ENDFOR
    \STATE Update $c_{a,i}(t) \; \forall a\in \setAgents ,\; i \in OS(a)$ \hfill $\triangleright$ Eq. \ref{eq:ci}
    \STATE Update $P_a \; \forall a \in \setAgents$ \hfill $\triangleright$ Eq. \ref{eq:update_pref}
    \STATE Eliminate pairs in $S_a$ \hfill $\triangleright$ Eq. \ref{eq:elim}
    \STATE $m_{\tau}= POSM(\{P_a\}_{a \in \setAgents})$
\ENDWHILE
\RETURN $m_{\tau}$
\end{algorithmic}
\end{algorithm}

To analyse the sample complexity of Algorithm~\ref{alg:elim}, we define the following reward gaps: For an agent $a \in \setAgents$ and $i \in OS(a)$. Let $\Delta_{a,i} = \min_{j \in OS(a) \setminus \{i\}} |\mu_{a,i} - \mu_{a,j}|$ denote the minimum reward gap separating $i$ from all other candidates, and define $\Delta_{a,i}^{\min} = \min\{\Delta_{a,i}, \Delta_{i,a}\}$ and $\Delta_{a,i}^{F} = \max(\Delta_{F}, \Delta_{a,i}^{\min})$.

\begin{theorem}\label{th:theorem_elim}
Algorithm~\ref{alg:elim} is a $\delta$-PCOS algorithm, and the number of sampled matchings is bounded by:
\begin{equation*}
    O\left(\max_{a \in \setAgents} \sum_{i \in OS(a)} \frac{32 \log(8KN/(\delta\Delta_{a,i}^{F}))}{(\Delta_{a,i}^{F})^2}  \right).
\end{equation*}
\end{theorem}

\subsection{Regret Analysis}
\label{sec:regret}
As previously discussed, almost all results bounding regret w.r.t. the optimal stable matching are based on the ETC framework. More specifically, in our setting we can employ an meta-algorithm that runs the E-PSM algorithm (Algorithm~\ref{alg:elim}) during the exploration phase, and subsequently commits to the identified matching for all remaining rounds.
We provide a bound for the following notion of regret:
\begin{equation}
\label{eq:regret}
    R(T) = \sum_{t=1}^T \mathbb{I}\{m_t \neq m^*\} \text{.}
\end{equation} 
\cite{basu2025competing} uses a similar binary regret for a matching to be stable. In addition, our metric upper-bounds the player-optimal stable regret defined by \cite{pmlr-v108-liu20c}, a metric that reflects market stability only when it is sublinear for every agent in the market.

\begin{theorem}\label{thm:regret}The E-PSM meta-algorithm with horizon $T$ and $\delta = 1/T$ has an expected regret $\mathbb{E}[R(T)]$ bounded by:$$ \mathcal{O}\left( \min \left\{ T, \max_{a\in\mathfrak{A}}\sum_{i\in OS(a)}\frac{32 \log(8KNT/\Delta_{a,i}^{F})}{(\Delta_{a,i}^{F})^{2}} \right\} \right) \text{.}$$\end{theorem}
Note that this minimum over the horizon cannot be avoided by any algorithm employing the Explore-Then-Commit framework. To our knowledge, this is the first regret bound with respect to the optimal stable matching under two-sided uncertainty that avoids a dependence on  $\Delta_{min}$.

\subsection{Extended Elimination Rule}
\label{sec:ext_elim}
We now present an extended version of our previous algorithm, which further eliminates pairs upon the identification of a \emph{super stable matching} (see Algorithm~\ref{alg:imp-elim} in the Appendix). Specifically, if at some time step the partial preference profile $\{P_a\}_{a \in \setAgents}$ admits a super stable matching $m_{\mathrm{ss}}$, we use the following elimination rule:
\begin{align*}
\label{eq:ext_elim_rule}
    S_p \leftarrow S_p \setminus \{a\} \quad \text{if} \quad m_{ss}(p) \succ_{P_p}  a \quad \forall \; p \in \setPlayer. \\
    S_a \leftarrow S_a \setminus \{p\} \quad \text{if} \quad p \succ_{P_a} m_{ss}(a) \quad \forall \; a \in \setArms.
\end{align*}
To show that this elimination still allows the identification of a valid profile, and thus a pervasive stable matching, we first show that one specific profile dependant on the optimal stable matching is always valid.
\begin{proposition}\label{prop: a-valid-profile}
    Let $F^*$ be a full ranking profile for agents $\setAgents = \setPlayer \cup \setArms$ with $\setPlayer$-optimal stable matching $\matching^*$. 
    Then the following partial profile $P$ is a valid profile for $F^*$:
\begin{itemize}
    \item $\begin{aligned}[t]
        \forall p\in\setPlayer, & P_p = \{(a, a') \in \setArms\times\setArms \mid a \succ a' \succeq_{F^*_p} \matching^*(p)\} \\
                 &\cup \{(a, a') \in \setArms\times\setArms \mid a \succeq_{F^*_p} \matching^*(p) \succ_{F^*_p} a'\}
          \end{aligned}$
    \item $\begin{aligned}[t]
        \forall a\in\setArms, &P_a = \{(p, p') \in \setArms\times\setArms \mid p' \prec p \preceq_{F^*_a} \matching^*(a)\} \\
                 &\cup \{(p, p') \in \setArms\times\setArms \mid p \preceq_{F^*_a} \matching^*(a) \prec_{F^*_a} p'\}
          \end{aligned}$
\end{itemize}
\end{proposition}
Now any partial profile compatible with profile $P$ as specified in Proposition~\ref{prop: a-valid-profile} is again a valid profile.
In particular, for a partial profile $P' = \{P'_a\}_{a \in \setAgents}$ that is compatible with full ranking profile $F$, $P^* = P \cup P' = \{P_a \cup P'_a\}_{a \in \setAgents}$ is a valid profile.
Assume that $P'$ admits a super stable matching $\matching$ that is not the PSM. Then $P^*$ does not contain:
\begin{itemize}
    \item $a \succ_{P^*_p} a'$ for $a, a' \in \setArms$ and $p \in \setPlayer$ with $\matching(p) \succ_{F_p} a,a'$, unless $a \succ_{P'_p} a'$:\\
    For optimal stable matching $\matching^*$ in $F$, we have $\matching^*(p) \succ_{F_p} \matching(p) \succ_{F_p} a,a'$. Thus $(a,a') \notin P_p$.
    \item $p \prec_{P^*_a} p'$ for $p, p' \in \setPlayer$ and $a \in \setArms$ with $\matching(a) \prec_{F_a} p,p'$, unless $p \prec_{P'_a} p'$: \\
    We have $\matching^*(a) \prec_{F_a} \matching(a) \prec_{F_a} p,p'$ and thus $(p,p') \notin P_a$.
\end{itemize}
 As these relations are not known in $P^*$, and $P^*$ is valid, we can eliminate agents as described above.

Note that just because Proposition~\ref{prop: a-valid-profile} guarantees that a valid profile (specifically $P^*$) can still be identified after eliminating agents form active pairs, this does not mean that the valid profiles that are still learnable attain the best-case gap $\Delta_{a,i}^{F}= \max(\Delta^{F},\Delta_{a,i}^{\min})$. We can therefore not transfer the sample complexity bound from Theorem~\ref{th:theorem_elim}. However, a super stable matching might be identifiable very fast, leading to early elimination of pairs and hence decreased sample complexity. We explore the sample complexity in our experiments to identify whether this trade-off between early elimination and restriction to valid profiles is beneficial.


\section{Lower bound}
\label{sec:lower_bound}
We now establish a non-explicit lower bound on sample complexity by adapting the analysis of \cite{jourdan2021efficient}.

For an instance $\nu \in \mathcal{I}$, we denote with $\matching^{\star}_{\nu}$ the optimal stable matching. To provide a lower bounds on the sample complexity, we define the set of alternative environments $\mathcal{I}_{\text{alt}}(\nu)$ as the collection of all environments in $\mathcal{I}$ where the optimal stable matching is different from that in $\nu$:
\begin{equation}
\mathcal{I}_{\text{alt}}(\nu) = \{ \nu' \in \mathcal{I} : \matching^{\star}_{\nu'} \neq \matching^{\star}_{\nu} \}
\end{equation}
Now we can define the transformed simplex $\mathcal{S}_{\mid \mathcal{M}\mid} = \{ \sum_{m \in \mathcal{M}} I_{m} w_m : w_m \in \Delta_\mathcal{M} \} \subset \mathbb{R}^{\setPlayer \times \setArms}$, where $I_{m} \subset \{0,1\}^{\setPlayer \times \setArms}$ represents the binary indicator vector of active pairs in matching $m$, and $w_m \in \Delta_\mathcal{M}$ denotes the proportion of rounds in which that specific matching is sampled.

\begin{theorem}
\label{th:lowerbound}
For any $\delta$-PACOS algorithm and any matching instance $\nu$,
\begin{equation}
    \mathbb{E}_{\nu}[\tau] \ge \ln{\frac{1}{2.4\delta}} c_{\star}^{-1}(\nu)
\end{equation}
where $c_{\star}(\nu)$ is the solution of the optimization problem:
\begin{equation*}
    c_{\star}(\nu) = \sup_{w\in \mathcal{S}_{\mid \mathcal{M}\mid}} \inf_{\nu' \in Alt(\nu)} \sum_{a \in \mathfrak{A}} \sum_{i \in OS(a)} w_{a,i} KL(\nu_{a,i}, \nu'_{a,i} ).
\end{equation*}
\end{theorem}

Although the bound is not explicit, it characterizes the intrinsic difficulty of the problem. In the MAB literature, this type of lower bound typically guides the design of optimal algorithms \cite{Degenne2019NonAsymptoticPE}. However, because solving the associated optimization problem is computationally hard, standard Track-and-Stop methods \cite{garivier2016optimal} cannot be directly adapted. Developing efficient approximations, such as adapting the methods of \cite{jourdan2021efficient}, is an interesting direction that remains open.

To provide a more concrete analysis, we establish an explicit lower bound for instances with global preferences in Appendix~\ref{app:lower_bound2}, similar to the approach of \cite{kaufmann2016complexity}, using arguments based on minimum valid preference profile.

\section{Simulations}
\label{sec:sim}
In this section, we present simulations to empirically evaluate the proposed algorithms on random instances for varying market sizes $n$, fixing $n=N=K$ agents on each side. The code to reproduce the presented results is provided online\footnote{\url{https://github.com/a-athanasopoulos/Probably-Correct-Optimal-Stable-Matching}}.

\textbf{Settings:} We perform two sets of experiments, covering both pure exploration and regret minimization, which we present in the following subsections.

\textbf{Instances:} To generate instances, we first construct the agents' preferences using random perturbations. We then assign the rewards for each agent $a \in \setAgents$ by sampling the reward gaps $\Delta_{a, i, i+1}$ between consecutive agents from a uniform distribution over $[\Delta_{min}, 1]$, where we set $\Delta_{min}= 0.2$.

\textbf{Algorithms:} To isolate the individual effects of the elimination strategy and the PSM-based stopping rule, we compare:
\begin{enumerate}[nosep, leftmargin=*]
    \item \textbf{U-$\Delta$:} The uniform sampling strategy of Section \ref{sec:etc} with known $\Delta_{\min}$.
    \item \textbf{U:} Uniform sampling of all pairs, stopping only when all confidence intervals separate (without $\Delta_{\min}$ knowledge).
    \item \textbf{U-PSM:} Uniform sampling of all pairs, terminating when a pervasive stable matching is identified.
    \item \textbf{E-NS:} The elimination algorithm without stopping rule, terminating only once all pairs are eliminated.
    \item \textbf{E-PSM:} The elimination algorithm of Section \ref{sec:elimination_algo}, terminating when a pervasive stable matching is identified.
    \item \textbf{EE-PSM (Extended PSM):} The extended elimination algorithm that further eliminates pairs after identifying a super stable matching as discussed in Section~\ref{sec:ext_elim}.
\end{enumerate}

\subsection{Pure Exploration}
\label{sim:pure_exp}

In Figure~\ref{fig:exp2}, we illustrate the sample complexity of the different approaches, averaged over 100 random instances, along with their standard deviations. For these experiments, we set the confidence parameter to $\delta=0.1$. Notably, all algorithms successfully identified the optimal stable matching in every run; this empirical correctness aligns with standard observations in the MAB literature \citep{pmlr-v35-jamieson14}. 

\begin{figure}[t]
    \centering
    \includegraphics[width=1\linewidth]{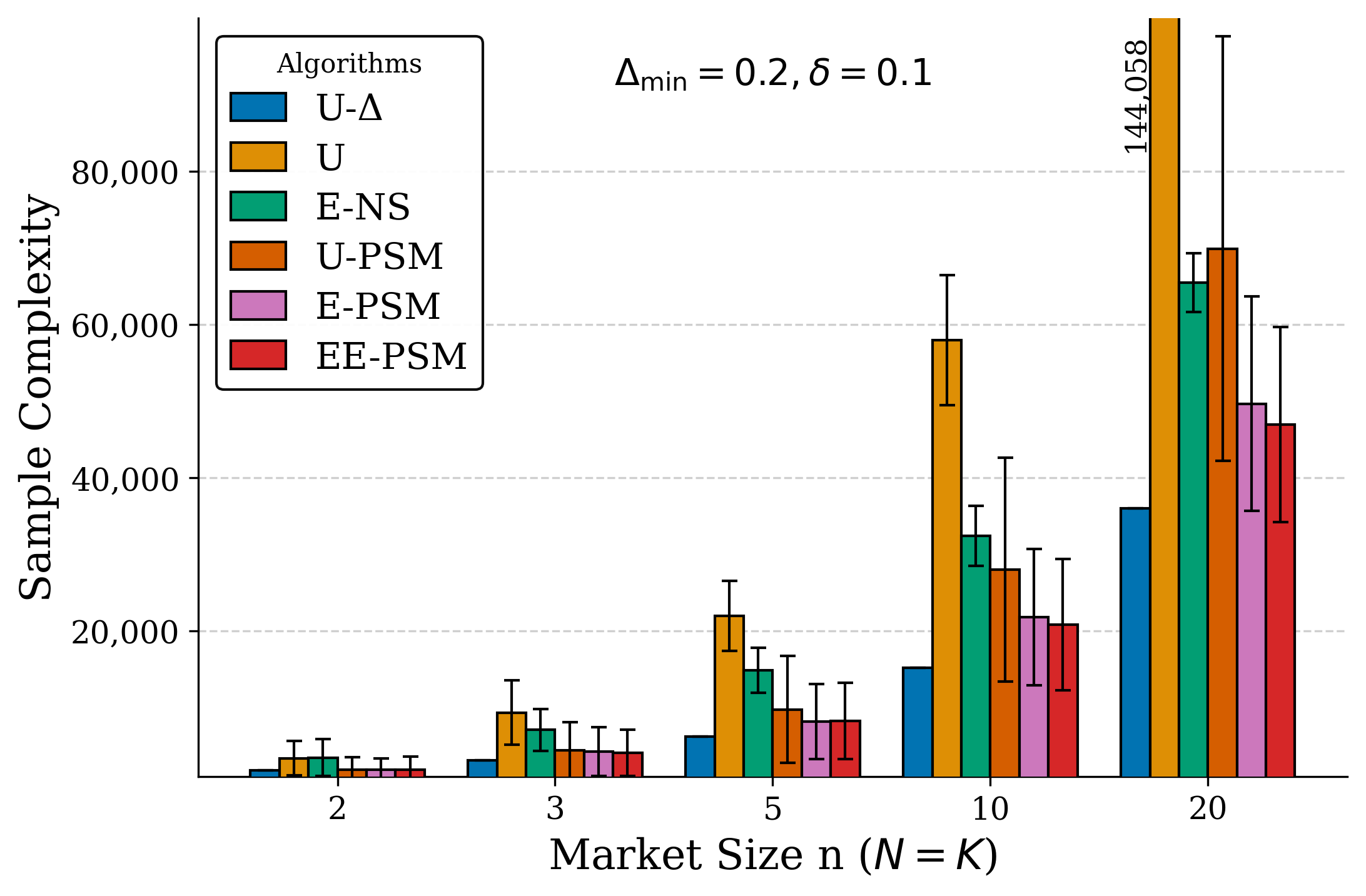}
    \caption{Sample complexity as a function of market size $n$.}
    \label{fig:exp2}
\end{figure}

Regarding sample complexity, we observe that the proposed algorithms, \textbf{E-PSM} and \textbf{EE-PSM}, significantly outperform all other baselines. Interestingly, we observe only a slight performance difference between these two specific algorithms on random instances. More specifically, upon further examining their behavior, we found that although the extended algorithm eliminated a large number of pairs, the size of the minimum matching cover did not decrease significantly, leading to near-identical sample complexities. 

Finally, the results demonstrate clear, independent benefits from both the elimination strategy and the early stopping rule. For example, comparing \textbf{U-PSM} to \textbf{E-PSM} isolates the positive effect of pair elimination, while comparing \textbf{E-NS} to \textbf{E-PSM} highlights the drastic impact of the stopping rule. Furthermore, \textbf{E-NS} outperforms \textbf{U-PSM} in $n=20$, demonstrating that the effect of the elimination rule is relatively larger on average as the market grows. To conclude, the results show that \textbf{U-PSM} exhibits larger variability across the random instances, suggesting a high sensitivity to specific reward structures. In contrast, the effect of the elimination rule in \textbf{E-NS} is much more stable.

Appendix~\ref{app:additional_exp} presents a similar experiment with different numbers of agents on each side, further demonstrating the effectiveness of our approach.

\subsection{Regret Minimization}
We further provide an ablation study for regret minimization. We measure regret using the cumulative market instability defined in Section~\ref{sec:regret}, Eq.~\ref{eq:regret}, which counts the number of rounds in which the selected matching differs from the optimal stable matching \citep{das2005two, basu2025competing}.

To clearly demonstrate the performance of our method against existing benchmarks, we focus on the more extensively studied one-sided uncertainty setting. This allows us to compare with prior regret-minimization algorithms, while our method naturally applies to the two-sided uncertainty setting considered in this paper. 

\textbf{Additional benchmark:} We specifically compare our approach with AETGSE, the algorithm of \cite{kong2024improved}, which combines adaptive elimination with the logic of GS algorithm. In the centralized setting, agents explore their currently valid partners in a coordinated round-robin manner, eliminating partners that are certified to be suboptimal. Once an agent identifies its most preferred valid partner, it stop exploring and temporarily commits to it. If this partner becomes unavailable (due to a higher-priority agent), the agent restarts the exploration process.

\begin{figure}[t]
    \centering
    \includegraphics[width=1\linewidth]{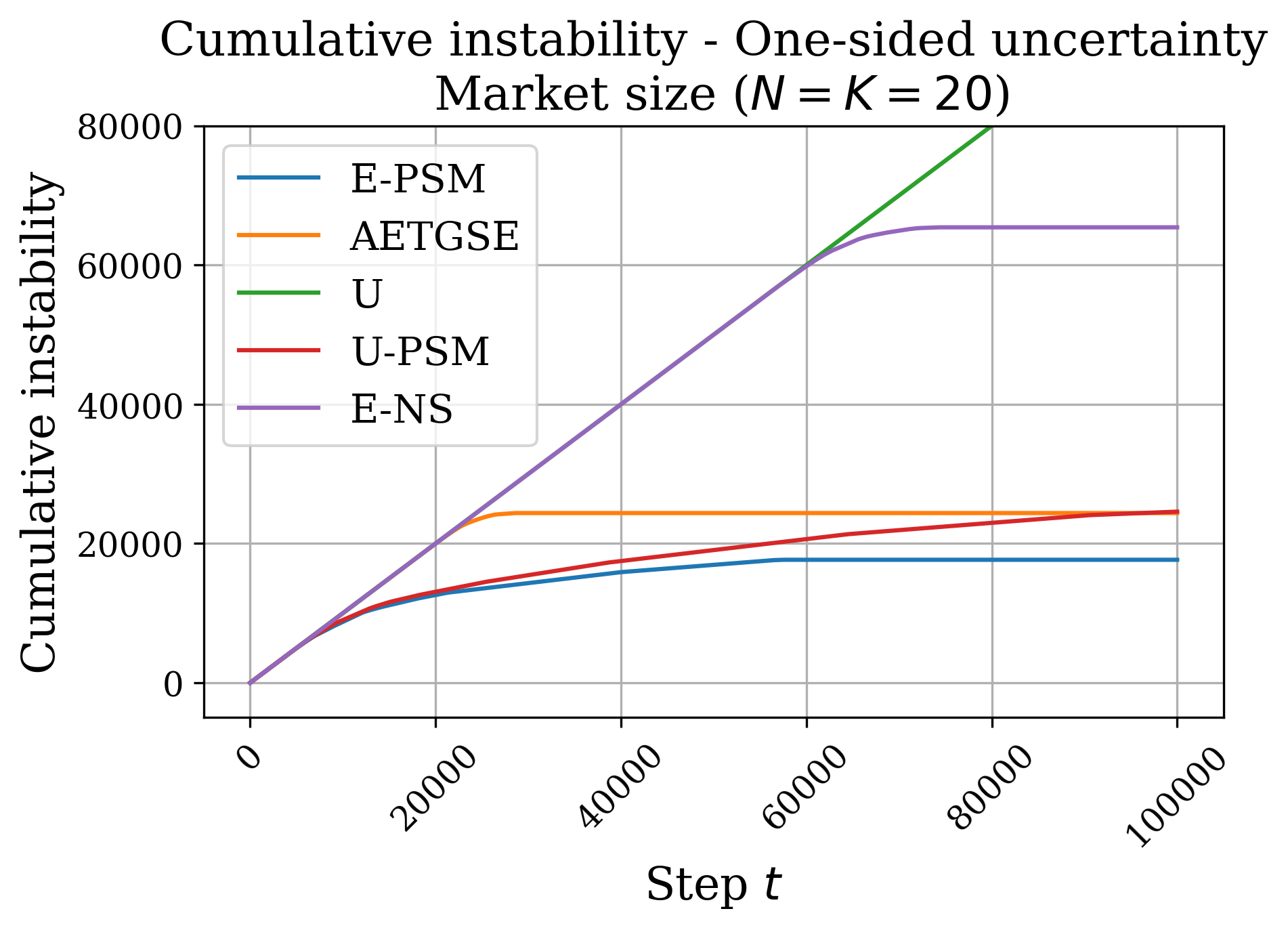}
    \caption{Cumulative regret}
    \label{fig:exp_regret}
\end{figure}

Figure~\ref{fig:exp_regret} presents the results over $50$ independent experiments for each algorithm, using random market instances with the number of agents on each side fixed to $K=N=20$. The results show that our elimination-based approach with a stopping rule based on pervasive stable matching (E-PSM) achieves lower regret in this one-sided uncertainty setting. Moreover, the uniform-sampling approach (U-PSM) interpolates between AETGS-E and E-PSM, suggesting that both the elimination procedure and our stopping rule contribute to improved learning performance.

Beyond the experimental comparison, our theoretical result yields a tighter upper bound than the analysis of AETGS-E, since the relevant gap depends on the minimum valid partial profile (see Section~\ref{sec:Setting}). Nevertheless, we highlight that AETGS-E may perform better on specific instances where agents do not need to restart exploration, due to its adaptive sampling approach. However, in more general instances, where exploration restarts, our algorithm achieves better performance, as demonstrated by our simulations.

\section{Conclusion}
In this work, we investigated a sequential learning problem in matching markets where both sides are initially unaware of their true preferences. Focusing on the pure exploration setting, our objective was to identify the optimal stable matching with high probability. Our primary contribution lies in leveraging partial knowledge of agents' preferences to stop exploration once a pervasive stable matching is found. We first demonstrated the performance gains of this approach on uniform exploration strategies, and subsequently introduced practical elimination algorithms along with a refined sample complexity analysis. Further, we extended our insights to the regret minimization setting and established theoretical lower bounds on the sample complexity.

For future work, it is worth proving theoretical sample-complexity guarantees for the extended algorithm from Section~\ref{sec:ext_elim}, and investigating whether we can further eliminate pairs after identifying a super stable matching, specifically by removing pairs that are not part of an exposed rotation \citep{gusfield1989stable}. In addition, adaptive algorithms, as in \citep{athanasopoulos2025probably, kong2024improved}, that do not uniformly sample active pairs are also a promising approach. Finally, designing optimal algorithms that match the lower bound established in our analysis remains an open problem.

\begin{acknowledgements}
    This project was partially supported by the Research Council of Norway through its Centre of Excellence: Integreat - The Norwegian Centre for knowledge-driven machine learning, project number 332645.
\end{acknowledgements}

\bibliography{uai2026-template}

\onecolumn

\title{Probably Correct Optimal Stable Matching: \\ The Two-Sided Uncertainty Case\\(Supplementary Material)}
\maketitle

\appendix

\section*{Appendix Outline}

In this appendix, we provide the detailed proofs and discussion omitted from the main text.
\begin{itemize}
    \item \textbf{Appendix~\ref{app:related_work}:} Provides an extended related work section along with a discussion on different learning objectives.
    \item \textbf{Appendix~\ref{app:etc}:} Contains the analysis of Theorems~\ref{th:theorem_etc1} and~\ref{th:theorem_etc2}  from Section~\ref{sec:etc}.
    \item \textbf{Appendix~\ref{app:elim}:} Contains the proof of Theorem~\ref{th:theorem_elim} from Section~\ref{sec:elim_psm}.
    \item \textbf{Appendix~\ref{app:regret}:} Contains the proof of  Theorem~\ref{thm:regret} from Section~\ref{sec:regret}.
    \item \textbf{Appendix~\ref{app:extended_elim}:} Contains the proof of Proposition~\ref{prop: a-valid-profile} from Section~\ref{sec:ext_elim}.
    \item \textbf{Appendix~\ref{app:lower_bound1}:} Contains the proof of Theorem~\ref{th:lowerbound} from Section~\ref{sec:lower_bound}.
    \item \textbf{Appendix~\ref{app:lower_bound2}:} Contains the analysis of the lower bound for specific instance and the proof of Lemma~\ref{lemma:global_pref}.
\end{itemize}

\section{Extended Related Work}
\label{app:related_work}
This section provides additional discussion on related work, expanding upon the overview presented in the main text.

\paragraph{MAB and pure exploration:} The stochastic MAB problem, introduced by \cite{thompson1933likelihood}, has been extensively studied in the past few decades, becoming a fundamental tool for analyzing online decision-making problems under uncertainty. A learner selects one action (or arm) in each round and receives a numeric reward drawn from an unknown distribution that reflects the utility of the action. In the classic version of regret minimization, the learner aims to maximize cumulative rewards by balancing exploration and exploitation: searching for optimal actions versus pulling known high-reward arms. Another popular setting is the pure exploration problem, where the learner's objective is to identify a good action regardless of the cost incurred in the learning process. This problem is usually studied in the fixed-confidence setting, where the learner aims to identify an optimal action with high probability using as few samples as possible \citep{even2002pac}. Initial research in the field explored uniform exploration and elimination-based strategies \cite{mannor2004sample, even2006action}, focusing on their correctness and sample complexity. More recent studies \cite{kaufmann2016complexity, garivier2016optimal} establish a lower bound on sample complexity, motivating the Track-and-Stop algorithm which achieves asymptotic optimality. In this paper, we study identifying the optimal stable matching in the fixed-confidence setting. For a comprehensive overview of MAB, we refer to the book by \cite{lattimore2020bandit}.

\paragraph{Stable matching problem:} Two-sided matching markets model situations in which two distinct sets of agents aim to match with each other based on individual preferences~\citep{gusfield1989stable, Roth_Sotomayor_1990, manlove2013algorithmics}. Classic examples include students applying to universities, workers matching with firms, and organ donors paired with patients~\citep{roth1984evolution}. 

In their fundamental work, \citet{gale_Shapley_1962} introduced the concept of stable matching, a notion of equilibrium in which no two agents have an incentive to deviate.
\citet{gale_Shapley_1962} proved that the set of stable matchings $\mathcal{S}$ is always non-empty by introducing the Gale–Shapley (GS) algorithm, which is based on sequential proposals from one side of the market to the other. The algorithm returns the (unique) \emph{optimal stable matching} $m^\star$ for the proposing side, in the sense that every proposer receives the best partner they can obtain in any stable matching. Conversely, the resulting matching is pessimal for the side receiving the proposals, providing each agent on that side with their least-preferred partner among all stable matchings.

A core component of our approach is utilizing the agents' \emph{partial preference} rankings to identify an optimal stable matching, thereby circumventing the need to learn their full preference orderings. This approach is motivated by the concept of \emph{pervasive stable matching}, introduced by \cite{rastegari2014reasoning} to reason about optimal stable matchings under partial information in \emph{non-learning} scenarios. Specifically, a partial preference profile of the agents admits a pervasive stable matching if that matching remains optimally stable under every possible full refinement of the given preferences. A closely related notion is that of a \emph{super stable matching}, initially introduced to address stability in instances with indifferent preferences \citep{IRVING1994261}.

\paragraph{Leaning in matching markets:} The problem of two-sided markets with unknown preferences has recently attracted considerable attention in various settings. \citeauthor{das2005two} were the first to introduce the marriage problem with unknown preferences on both sides of the market, framing it as a MAB problem. In their work \cite{das2005two}, they empirically studied several algorithms in specific preference settings. A follow-up by \citet{pmlr-v108-liu20c} introduced the notions of \emph{Player-Optimal} and \emph{Player-Pessimal} stable regret. Focusing on a setting where only one side of the market has unknown preferences, they analyzed an Explore-Then-Commit algorithm, proving sublinear bounds for both regret notions. They also examined an Upper Confidence Bound approach, establishing sublinear bounds for Player-Pessimal stable regret; however, they proved an impossibility result for bounding Player-Optimal regret.

Later studies closed this gap by achieving sublinear player-optimal stable regret \citep{sankararaman2021dominate, basu2021beyond, kong2023player}, though they primarily focused on one-sided uncertainty and specific market structures. More recent studies have begun to consider settings where both sides of the market have uncertain preferences \citep{pagare2024explore, pmlr-v244-zhang24b}. However, these works employ Explore-Then-Commit type algorithms. In contrast, this work adopts a pure exploration framework that is more naturally suited to the analysis of such algorithms and enables refined results on regret with respect to the \emph{optimal} stable matching.

The pure exploration objective of identifying an \emph{optimal stable matching} with high probability was introduced by \cite{athanasopoulos2025probably}, while the work of \cite{NEURIPS2024_7fc0aac0} focuses on identifying \emph{any} stable matching. Both of these studies propose elimination-based algorithms restricted to one-sided uncertainty. We extend this literature by tackling a two-sided learning environment, providing refined arguments for the analysis of elimination algorithms. 

Conceptually, the closest work to ours is \cite{basu2025competing}, where the authors leverage partial preferences and the concept of super stability to bound regret in sequential learning scenarios. In contrast, our approach specifically relies on pervasive stable matchings, instead of the super stable matching. This distinction is crucial, as our objective is to identify the \emph{optimal} stable matching, rather than \emph{any} stable matching.

Prior work also considers decentralized settings \citep{liu2021bandit, kong2022thompson, sankararaman2021dominate, basu2021beyond, pagare2024explore, pmlr-v244-zhang24b}, models with transferable-utilities \citep{cen2022regret, jagadeesan2021learning, basu2025competing}, and the setting with many-to-one matchings \citep{saha2024altruistic, wang2022bandit, kong2024improved}.

\paragraph{Discussion on alternative objectives:}
Related work typically characterizes algorithmic performance based on the cumulative rewards collected by individual agents. In particular, most existing approaches follow the notion of agent-optimal (or pessimal) stable regret, which compares accumulated rewards against a baseline that always selects the optimal stable matching, following the definitions in the seminal work of \cite{pmlr-v108-liu20c}. This cumulative metric reflects stability only if the regret is sublinear for every agent in the market. Furthermore, this objective can be satisfied by a composition of unstable matchings over the time horizon—similar to the concept of fractional stability over probabilistic matchings that may individually be unstable \citep{caragiannis2019stable}. Another notion of regret, more compatible with our objective, is binary stable regret, which evaluates the stability of the matching at each discrete time step \citep{das2005two, basu2025competing}. Current studies in this area typically consider a matching to be stable without specifically requiring the identification of the optimal stable matching. Additionally, algorithms addressing these different objectives often rely on the Explore-then-Commit (ETC) framework, bounding the probability of error through sufficient exploration \citep{pmlr-v108-liu20c, pagare2024explore, basu2025competing, kong2023player}. Notable exceptions to these ETC-based approaches include analyses of UCB-style algorithms \citep{pmlr-v108-liu20c, jagadeesan2021learning2} and Thompson sampling \citep{kong2022thompson}, which do not guarantee sublinear regret with respect to the optimal stable matching. Ultimately, different objectives prioritize different criteria; we emphasize that our framework is specifically designed for applications where the primary goal is to learn the optimal stable matching itself, rather than merely achieving stability in a cumulative or general sense.

\section{Analysis of Section~\ref{sec:etc}}
\label{app:etc}
In this section, we provide the proofs for the sample complexity of our uniform exploration strategies.

\subsection{Proof of Theorem~\ref{th:theorem_etc1}}
\begin{proof}
We begin the proof by characterizing the sample complexity of the algorithm. We recall that the sample complexity is defined as the total number of matchings performed by the algorithm until termination.

\textbf{Sample Complexity} \\
The sample complexity follows directly from the design of Algorithm~\ref{alg:etc}. In each exploration round $h$, the algorithm ensures that every available pair $(p, a) \in \mathcal{P} \times \mathcal{A}$ is sampled exactly once, by sampling a sequence of $n=\max\{K, N\}$ matchings. This is achieved by sampling the matching from a minimal matching cover $\mathfrak{M}(\mathcal{P} \times \mathcal{A})$. By Kőnig's Theorem, the size of this cover is equal to the maximum degree of the underlying complete bipartite graph, which is $\max\{K, N\}$

Given that the total number of exploration rounds is $h$, the total number of matchings sampled is:
\begin{equation}
    \tau = h \cdot \max\{K, N\} = \left\lceil \frac{8 \ln(2KN/\delta)}{\Delta_{\min}^2} \right\rceil \cdot \max\{K, N\}.
\end{equation}
Thus, the sample complexity is $O\left(\max\{K, N\} \frac{\ln(KN/\delta)}{\Delta_{\min}^2}\right)$.

\textbf{Correctness} \\
We now prove that the algorithm returns the optimal stable matching with probability at least \(1-\delta\), i.e.,
\(\Pr(m_{\tau}=m^\star)\ge 1-\delta\).
If all agents' preference orderings are estimated correctly, then running GS on the estimated profile returns the true optimal stable matching; therefore,
\[
\Pr\!\left( m_{\tau} = m^{\star} \right)
\;\ge\;
\Pr\!\left( \hat{F}_a = F_a \;\; \forall a \in \setAgents \right)
\;=\;
1-\Pr\!\left(\exists a\in\setAgents:\hat{F}_a\neq F_a\right).
\]
Thus, it suffices to show \(\Pr(\exists a\in\setAgents:\hat{F}_a\neq F_a)\le \delta\).

Now fix an agent $a\in \setAgents$. A sufficient condition for every empirical ranking $\hat{F}_a$ to match $F_a$, is that for every agent $i \in OS(a)$, $|\hat{\mu}_{a,i} - \mu_{a,i}| < \frac{\Delta_{\min}}{2}$ holds. Thus for every pair $(i,j)$ where $\mu_{a,i}>\mu_{a,j}$ we have that:
\begin{equation}
\hat{\mu}_{a,i} > \mu_{a,i} - \frac{\Delta_{\min}}{2} \ge \mu_{a,j} + \frac{\Delta_{\min}}{2} > \hat{\mu}_{a,j},
\end{equation}
ensuring that $\hat{\mu}_{a,i} > \hat{\mu}_{a,j} $ using the definition of $\Delta_{min}$. 

By the construction of Algorithm 1, each pair $(a, i) \in \setAgents \times OS(a)$ is sampled exactly $h$ times. Since the rewards are $1$-sub-Gaussian, the empirical mean $\hat{\mu}_{a,i}$ is $(1/h)$-sub-Gaussian. Applying the standard sub-Gaussian tail bound:
\begin{equation}
\Pr(\hat{F}_a \neq F_a) = \Pr \left( |\hat{\mu}_{a,i} - \mu_{a,i}| \ge \frac{\Delta_{\min}}{2} \right) \le 2 \exp\left( - \frac{h (\Delta_{\min}/2)^2}{2} \right) = 2 \exp\left( - \frac{h \Delta_{\min}^2}{8} \right).
\end{equation}

By applying a union bound over all $KN$ pairs in the bipartite market:
\begin{equation}\mathbb{P} \left( \exists a \in \setAgents : \hat{F}_a \neq F_a \right) \le \sum_{a\in \setAgents} \mathbb{P} \left(\hat{F}_a \neq F_a \right)  \le 2KN \exp\left( - \frac{h \Delta_{\min}^2}{8} \right).\end{equation}

Now setting  $h=\left\lceil \frac{8\ln(2KN/\delta)}{\Delta_{\min}^2}\right\rceil$, results that $\mathbb{P} \left( \exists a \in \setAgents : \hat{F}_a \neq F_a \right) \leq \delta$ and thus:
\begin{equation}
    \Pr\!\left( m_{\tau} = m^{\star} \right) < 1 -\delta
\end{equation}
That concludes the proof.
\end{proof}

\subsection{Proof of
Theorem~\ref{th:theorem_etc2} }
\begin{proof}
Now we continue with a similar analysis for the second version of the algorithm, using arguments regarding the preference profiles $\mathcal{P}_{\mathrm{valid}}(F_{\mathfrak{A}}^{\nu})$.

\textbf{Sample Complexity} \\
The sample complexity follows again directly from the definition of the algorithm, where in each exploration round $h$ we sample $\max\{K, N\}$ number of matchings. Setting $h = \lceil \frac{8 \ln(2KN/\delta)}{\Delta_{F}^2} \rceil$ results in a sample complexity of $O\left(\max\{K, N\} \frac{\ln(KN/\delta)}{\Delta_{F}^2}\right)$.

\textbf{Correctness} \\
We now prove that the algorithm returns the optimal stable matching with probability at least \(1-\delta\). Instead of learning the full preferences $F_A$ for the agents, the algorithm samples enough to learn a valid preference profile $P_A \in \mathcal{P}_{\mathrm{valid}}(F_{A})$. By definition, a valid preference profile admits a \emph{pervasive stable matching}, and thus, running the $GS(\hat{F}_A)$ algorithm
with any realization of the preferences $P_A$, i.e., 
 $ \hat{F}_A \in \mathcal{F}(P_{A}) = \{F_{A} \mid P_{A} \mapsto F_{A}\}$ results in the true optimal stable matching $m^\star$. More precisely, we have that:
\begin{equation}
\Pr(m_{\tau}=m^\star) \ge \Pr(\exists P_{A} \in \mathcal{P}_{valid} : \hat{F}_{A} \in \mathcal{F}(P_{A})) \ge \Pr(\hat{F}_{A} \in \mathcal{F}(P^{\star})) .
\end{equation}
where $P^{\star}_A = \arg\max_{P\in \mathcal{P}_{\mathrm{valid}}(F_A)} \Delta_{\min}(P)$ be the valid preference profile that attains the admissible gap $\Delta_{\min}(P^{\star}_A) = \Delta_F$.

A sufficient condition for the empirical profile $\hat{F}_{A}$ to refine $P_{A}^{*}$ is that for every agent $a \in \mathfrak{A}$ and every $i \in OS(a)$, the condition $|\hat{\mu}_{a,i} - \mu_{a,i}| < \frac{\Delta_{F}}{2}$ holds. Thus, for every agent $a \in \mathfrak{A}$ and every pair $(i, j) \in P_{a}^{*}$ where $\mu_{a,i} \ge \mu_{a,j} $, we have that:
$$ \hat{\mu}_{a,i} > \mu_{a,i} - \frac{\Delta_{F}}{2} \ge \mu_{a,j} + \frac{\Delta_{F}}{2} > \hat{\mu}_{a,j}, $$
ensuring that the relative ordering in $P_{A}^{*}$ is correctly estimated.

By the construction of Algorithm 1, each pair $(a,i) \in \mathfrak{A} \times OS(a)$ is sampled exactly $h$ times. Since the rewards are 1-sub-Gaussian, the empirical mean $\hat{\mu}_{a,i}$ is $(1/h)$-sub-Gaussian. Applying the standard sub-Gaussian tail bound:

$$ Pr\left(|\hat{\mu}_{a,i} - \mu_{a,i}| \ge \frac{\Delta_{F}}{2}\right) \le 2 \exp\left(-\frac{h(\Delta_{F}/2)^{2}}{2}\right) = 2 \exp\left(-\frac{h\Delta_{F}^{2}}{8}\right) $$

By applying a union bound over all $KN$ pairs in the bipartite market:
$$ \mathbb{P}\left(\exists a \in \mathfrak{A}, i \in OS(a) : |\hat{\mu}_{a,i} - \mu_{a,i}| \ge \frac{\Delta_{F}}{2}\right) \le \sum_{a \in \mathfrak{A}} \sum_{i \in OS(a)} \mathbb{P}\left(|\hat{\mu}_{a,i} - \mu_{a,i}| \ge \frac{\Delta_{F}}{2}\right) \le 2KN \exp\left(-\frac{h\Delta_{F}^{2}}{8}\right) $$
Now setting $h = \left\lceil \frac{8 \ln(2KN/\delta)}{\Delta_{F}^{2}} \right\rceil$ ensures that the probability of failing to correctly estimate the valid preference profile is bounded by $\delta$, and thus:
$$ Pr(m_{\tau} = m^{*}) \ge 1 - \delta $$

\end{proof}

\section{Analysis of Section~\ref{sec:elimination_algo}}
\label{app:elim}

\subsection{Proof of Theorem~\ref{th:theorem_elim}}
\begin{proof} \textbf{ } \\
    
\textbf{Correctness argument} \\
We first present a formal argument that the elimination algorithm identifies the optimal stable matching with probability at least $1-\delta$. 

To begin, the algorithm is guaranteed to terminate because the confidence radius $c_{a,i}(t)$ shrinks toward zero as the number of samples $t$ increases. Since the algorithm continues to sample every non-eliminated pair in each round, the confidence intervals will eventually shrink enough to resolve the relative ordering of every pair. Consequently, the algorithm will eventually eliminate all pairs satisfying the stopping condition.

Now, let $\mathcal{E}$ be the event that the true means $\mu_{a,i}$ lie within the confidence intervals for all estimates $\hat{\mu}_{a,i}(t)$ at all times $t$. Formally, $\mathcal{E}$ is the intersection of all individual events $\mathcal{E}_{t,a,i}$:
\begin{equation}
    \mathcal{E} = \bigcap_{t=1}^{\infty} \bigcap_{a\in \setAgents} \bigcap_{i \in OS(a)} \mathcal{E}_{t,a,i},
\end{equation}
where $\mathcal{E}_{t,a,i}$ denotes the event that the true mean $\mu_{a,i}$ lies within the confidence interval at time $t$:
\begin{equation}
    \mathcal{E}_{t,a,i} = \left\{ \left| \hat{\mu}_{a,i}(t) - \mu_{a,i} \right| \leq c_{a,i}(t) \right\}.
\end{equation}

In Lemma~\ref{lem:good_event_probability} we prove that the event $\mathcal{E}$ holds with probability at least $1-\delta$, i.e. $\Pr(\mathcal{E}) \geq 1- \delta$.

Under event $\mathcal{E}$, the algorithm only adds correct preference relations to $P_a$ because $LCB_{a,i}(t) > UCB_{a,j}(t)$ implies $\mu_{a,i} > \mu_{a,j}$. Since the algorithm terminates with a valid preference profile $\hat{P}_A \in \mathcal{P}_{\mathrm{valid}}(F_A)$, it returns the correct optimal stable matching $m^\star$

To conclude, under the good event $\mathcal{E}$, which holds with probability at least $1-\delta$, the algorithm returns the optimal stable matching.

\subsection{Sample complexity}
We now continue the analysis on sample complexity in terms of matching samples. Recall that in each round $t$, we sample every non-eliminated pair $S_t$ by sampling $n_t$ matchings from a minimum matching cover $\mathfrak{M}(S_t)$.

Let $\tau_{a,i}$ denote the number of rounds required to eliminate agent $i \in OS(a)$ from the active set $S_{a}$ of agent $a\in \setAgents$. To eliminate agent $i$, we must successfully distinguish it from every other agent $j \in OS(a)\setminus \{i\}$. Therefore, the elimination time is dominated by the time required to separate $i$ from its closest competitor, which corresponds to the minimum reward gap $\Delta_{a,i} = \min_{j \in OS(a) \setminus \{i\}} |\mu_{a,i} - \mu_{a,j}|$. By applying Lemma~\ref{lem:elimination_time} we have that:
\begin{equation}
\label{eq:tai}
\tau_{a,i} \leq \frac{32 \log(8KN/(\delta\Delta_{a,i}))}{\Delta_{a,i}^2} \text{.}
\end{equation}

Since eliminating a pair $(p,a)\in \setPlayer\times\setArms$ requires removing $p$ from $S_a$ and $a$ from $S_p$, the maximal number of rounds is $t_{p,a} = \max\{\tau_{p,a}, \tau_{a,p}\} \leq \frac{32 \log(8KN/(\delta\Delta_{p,a}^{\min}))}{(\Delta_{p,a}^{\min})^2}$, where $\Delta_{p,a}^{\min} = \min\{\Delta_{p,a}, \Delta_{a,p}\}$.

\begin{figure}[h]
    \centering
    \begin{tikzpicture}[
        thick,
        agent/.style={circle, draw, minimum size=1cm, inner sep=0pt, fill=blue!10},
        arm/.style={circle, draw, minimum size=1cm, inner sep=0pt, fill=red!10},
        connection/.style={-, thick}
    ]

        \node[agent] (p1) at (0, 3) {$p_1$};
        \node[agent] (p2) at (0, 1.5) {$p_2$};
        \node[agent] (p3) at (0, 0) {$p_3$};
        
        \node[above=0.5cm of p1, font=\bfseries] {Set $P$};

        \node[arm] (a1) at (3, 3) {$a_1$};
        \node[arm] (a2) at (3, 1.5) {$a_2$};
        \node[arm] (a3) at (3, 0) {$a_3$};

        \node[above=0.5cm of a1, font=\bfseries] {Set $A$};

        
        \draw[connection] (p1) -- node[midway, above] {$t_{p_1,a_1}$} (a1);%
        \draw[connection] (p1) -- (a2);
        \draw[connection] (p1) -- (a3);
        \draw[connection] (p2) -- (a1);
        \draw[connection] (p2) -- (a2);
        \draw[connection] (p2) -- (a3);
        \draw[connection] (p3) -- (a1);
        \draw[connection] (p3) -- (a2);
        \draw[connection] (p3) -- (a3);

    \end{tikzpicture}
    \caption{Illustration of the bipartite graph $G$.}
    \label{fig:bipartite_graph1}
\end{figure}
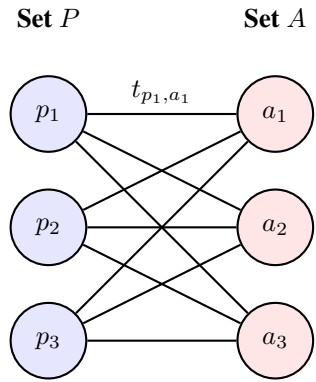

\paragraph{A. Sample complexity Without Stopping Rule} \textit{} \\
Next, we give the sample complexity ignoring the stopping rule where the algorithm terminates if the partial preferences admit a pervasive stable matching.

Given the sampling requirements $t_{p,a}$ for each pair of agents $(p,a) \in \setPlayer \times \setArms$ (where $t_{p,a} = t_{a,p}$) , we construct a bipartite multigraph $G_M = (V = \setPlayer \cup \setArms, E = \setPlayer \times \setArms)$ illustrated in Figure \ref{fig:bipartite_graph1}, where $t_{p,a}$ denotes the edge multiplicity for the pair $(p,a)$. By Kőnig's Theorem, there exists a decomposition of these edges into $deg(G_M)$ matchings, where $deg(G_M)$ is the maximum degree of the multigraph. Consequently, the number of matchings is at most:
\begin{equation}
    O(\max_{a \in \setAgents} \sum_{i \in OS(a)} t_{a,i})
\end{equation}
since $\sum_{i \in OS(a)} t_{a,i}$ denotes the degree of agent $a$ in the multigraph.

As a note, we use this variant of the algorithm to demonstrate the isolated effect of the elimination rule in the experiments section.

\paragraph{B. Sample complexity With Stopping Rule} \textit{} \\
We now perform the analysis considering the stopping rule. The algorithm terminates once it correctly estimates a partial preference profile $\hat{P}_A$ that admits a pervasive stable matching, i.e., $\hat{P}_A \in \mathcal{P}_{\mathrm{valid}}(F_A)$. By definition, there exists at least one valid preference profile $P_A^F \in \mathcal{P}_{\mathrm{valid}}(F_A)$ that admits the maximum possible minimum gap $\Delta_{F}$

Since the algorithm uniformly samples every non-eliminated pair in each round, it is guaranteed to correctly estimate all preference relations in this optimal target profile $P_A^F$ once all reward gaps of size at least $\Delta_F$ are resolved. By applying Lemma~\ref{lem:elimination_time}, this requires at most $\tau^F$ rounds, where
\begin{equation}
\tau^{F} \leq \frac{32 \log(8KN/(\delta\Delta_{F}))}{\Delta_{F}^2} \text{.}
\end{equation}

Consequently, any pair $(a,i)$ is sampled only until it is individually eliminated from the active sets, or until the algorithm successfully identifies $P_A^F$ and terminates globally. Thus, the number of times pair $(a,i)$ is sampled is strictly bounded by $t^{F}_{a,i} = \min(t_{a,i}, \tau^{F})$. Using the individual elimination bound for $t_{a,i}$ (Eq.~\ref{eq:tai}) and the global termination bound $\tau^F$, we have:
\begin{equation}
t^{F}_{a,i} \leq  \frac{32 \log(8KN/(\delta\Delta_{a,i}^{F}))}{(\Delta_{a,i}^{F})^2},
\end{equation}
where $\Delta_{a,i}^{F} = \max(\Delta_{F}, \Delta_{a,i}^{\min})$.

Following the logic of the previous analysis, we consider the bipartite multigraph $G_M$ with edge multiplicities $t^{F}_{a,i}$. By Kőnig's Theorem, there exists a decomposition of these edges into $deg(G_M)$ matchings, where $deg(G_M)$ is the maximum degree. Therefore, the total number of matching samples is at most:
\begin{equation}
O\left(\max_{a \in \setAgents} \sum_{i \in OS(a)} t^{F}_{a,i}\right)
\end{equation}

To conclude, we prove that with probability at least $1-\delta$, the algorithm terminates with the correct optimal stable matching, and that with probability at least $1-\delta$, the sample complexity is at most:
\begin{equation}
O\left(\max_{a \in \setAgents} \sum_{i \in OS(a)} \frac{32 \log(8KN/(\delta\Delta_{a,i}^{F}))}{(\Delta_{a,i}^{F})^2},\right)
\end{equation}
where $\Delta_{a,i}^{F} = \max(\Delta_{F}, \Delta_{a,i}^{\min})$.
\end{proof}

\subsection{useful lemmas}
\begin{lemma}
\label{lem:good_event_probability}
Let $\mathcal{E}$ be the event that for all time steps $t \geq 1$, all agents $a \in \setAgents$, and all agents $i \in OS(a)$, the empirical mean estimates lie within the confidence intervals defined by $c_{a,i}(t) = \sqrt{\frac{2\ln(8KNt^2/\delta)}{t}}$.
Then, the probability of this event holds satisfies:
\begin{equation}
    \Pr(\mathcal{E}) \geq 1 - \delta.
\end{equation}
\end{lemma}

\begin{proof}
We bound the probability of the bad event $\mathcal{E}^c$, defined as the event where there exists a time step $t$, an agent $a$, and a candidate $i$ such that the empirical mean falls outside the confidence interval: $\left| \hat{\mu}_{a,i}(t) - \mu_{a,i} \right| \ge c_{a,i}(t)$.

Using the union bound over all time steps $t \ge 1$, all agents $a \in \setAgents$, and all candidates $i \in OS(a)$, we have:
\begin{equation}
    \Pr(\mathcal{E}^c) \leq \sum_{t=1}^{\infty} \sum_{a \in \setAgents} \sum_{i \in OS(a)} \Pr(\mathcal{E}_{t,a,i}^c).
\end{equation}

Since the rewards are $1$-sub-Gaussian and we use the confidence radius $c_{a,i}(t) = \sqrt{\frac{2\ln(8KNt^2/\delta)}{t}}$, applying the concentration bound for sub-Gaussian random variables yields (Corollary 5.5. \citep{Lattimore_bool}):
\begin{align}
    \Pr(\mathcal{E}_{t,a,i}^c) 
    &= \Pr\left( \left| \hat{\mu}_{a,i}(t) - \mu_{a,i} \right| \ge c_{a,i}(t) \right) \\
    &\leq 2 \exp\left(  \frac{-tc_{a,i}(t)^2}{2 } \right)  \\
    &= \frac{\delta}{4KNt^2}.
\end{align}
Substituting this back into the union bound sum:
\begin{equation}
    \Pr(\mathcal{E}^c) \leq \sum_{t=1}^{\infty} \sum_{a \in \setAgents} \sum_{i \in OS(a)}  \frac{\delta}{4KNt^2} \leq \delta \sum_{t=1}^{\infty}  \frac{1}{2t^2} \leq \delta .
\end{equation}

Finally, $\Pr(\mathcal{E})= 1 - \Pr(\mathcal{E}^c) = 1 - \delta$ witch concludes the proof.
\end{proof}

\begin{lemma}
\label{lem:elimination_time}
Fix an agent $a \in \setAgents$ and two distinct agents $i, j \in OS(a)$. Let $\Delta_{a,i,j} = |\mu_{a,i} - \mu_{a,j}|$ be the true reward gap between $i$ and $j$. Under the event $\mathcal{E}$, the number of rounds $\tau_{a, i, j}$ required to separate agent $i$ from agent $j$ is at most:
\begin{equation}
    \tau_{a,i,j} \leq \frac{32 \log(8KN/(\delta\Delta_{a,i,j}))}{\Delta_{a,i,j}^2} \text{.}
\end{equation}
\end{lemma}

\begin{proof}
Assume, without loss of generality, that $\mu_{a,i} > \mu_{a,j}$. The algorithm successfully separates agent $i$ from agent $j$ at round $t$ if their confidence intervals become strictly disjoint, which occurs when $LCB_{a,i}(t) > UCB_{a,j}(t)$. 

Under the good event $\mathcal{E}$, the empirical means are bounded by their confidence radii $c(t)$, meaning $\hat{\mu}_{a,i}(t) \geq \mu_{a,i} - c(t)$ and $\hat{\mu}_{a,j}(t) \leq \mu_{a,j} + c(t)$. 
Substituting these into the lower and upper bounds, we get:
\begin{align}
    LCB_{a,i}(t) &= \hat{\mu}_{a,i}(t) - c(t) \geq \mu_{a,i} - 2c(t) \\
    UCB_{a,j}(t) &= \hat{\mu}_{a,j}(t) + c(t) \leq \mu_{a,j} + 2c(t)
\end{align}
Therefore, a sufficient condition to guarantee $LCB_{a,i}(t) > UCB_{a,j}(t)$ is:
\begin{equation}
    \mu_{a,i} - 2c(t) > \mu_{a,j} + 2c(t) \implies \Delta_{a,i,j} > 4c(t) \text{.}
\end{equation}
Substituting the confidence radius $c(t) = \sqrt{\frac{2\ln(8KNt^2/\delta)}{t}}$ into our sufficient condition yields:
\begin{equation}
    \Delta_{a,i,j} > 4 \sqrt{\frac{2\ln(8KNt^2/\delta)}{t}} = \sqrt{\frac{32 \ln(8KNt^2/\delta)}{t}} \text{.}
\end{equation}
Squaring both sides and solving for $t$, we find that the intervals are guaranteed to separate for any round $t$ satisfying:
\begin{equation}
    t > \frac{32}{\Delta_{a,i,j}^2} \log\left(\frac{8KNt^2}{\delta}\right) \text{.}
\end{equation}
Using standard algebraic bounds for isolating $t$ in action-elimination algorithms, this inequality is satisfied for a number of rounds at most:
\begin{equation}
    \tau_{a,i,j} \leq \frac{32 \log(8KN/(\delta\Delta_{a,i,j}))}{\Delta_{a,i,j}^2} \text{.}
\end{equation}
This concludes the proof.
\end{proof}

\section{Proof of Theorem~\ref{thm:regret}}
\label{app:regret}
\begin{proof}
We consider the meta-algorithm that runs the E-PSM algorithm during the exploration phase. If the algorithm terminates at round $\tau < T$, it commits to the identified matching $m_{\tau}$ for the remainder of the horizon $T$.

Let $\mathcal{E}$ denote the good event where the algorithm outputs the correct optimal stable matching, i.e., $\mathcal{E} = \{ m_{\tau} = m^{\star} \}$. By Theorem 5.1, the E-PSM algorithm is $\delta$-PCOS. Thus, the probability of success is bounded by $\mathbb{P}(\mathcal{E}) \ge 1-\delta$, and the probability of failure is bounded by $\mathbb{P}(\mathcal{E}^c) \le \delta$.

Using the law of total expectation, we can decompose the expected regret as follows:
\begin{align}
    \mathbb{E}[R(T)] &= \mathbb{E}[R(T) \mid \mathcal{E}]\mathbb{P}(\mathcal{E}) + \mathbb{E}[R(T) \mid \mathcal{E}^c]\mathbb{P}(\mathcal{E}^c) \\
    &\le \mathbb{E}[R(T) \mid \mathcal{E}] \cdot 1 + T \cdot \delta \text{,}
\end{align}
where we used the trivial bound $\mathbb{P}(\mathcal{E}) \le 1$.

Conditioned on the good event $\mathcal{E}$, the algorithm correctly identifies $m^\star$ and incurs zero regret during the commitment phase. Since there is no a  guarantee that the required exploration time $\tau \le T$, we have:
\begin{equation}
    \mathbb{E}[R(T) \mid \mathcal{E}] = \mathbb{E}[\min(\tau, T) \mid \mathcal{E}] \text{.}
\end{equation}

Setting our confidence parameter to $\delta = 1/T$ yields to:
\begin{equation}
    \mathbb{E}[R(T)] \le \mathbb{E}[\min(\tau, T)\mid \mathcal{E}] + 1 \text{.}
\end{equation}

Finally, substituting the bound sample complexity derived in Theorem~\ref{th:theorem_elim}, that hold under the good event, we have that:
\begin{equation}
    \mathbb{E}[R(T)] \le \mathcal{O}\left( \min \left\{ T, \max_{a\in\mathfrak{A}}\sum_{i\in OS(a)}\frac{32 \log(8KNT/\Delta_{a,i}^{F})}{(\Delta_{a,i}^{F})^{2}} \right\} \right) \text{.}
\end{equation}
which concludes the proof.
\end{proof}

\section{Extended Elimination Algorithm}
\label{app:extended_elim}

\begin{algorithm}[t]
\caption{EE-PSM: Extended Elimination PSM}\label{alg:imp-elim}
\begin{algorithmic}[1]
\REQUIRE $\delta$, $\setPlayer$, $\setArms$
\STATE $S_a \leftarrow OS(a) \quad \forall a \in \setAgents$ \hfill $\triangleright$ Active pairs
\STATE $P_a \leftarrow \emptyset \quad \forall a \in \setAgents$ \hfill $\triangleright$ Partial rank
\STATE $  m_{\tau}= POSM(\{P_a\}_{a \in \setAgents})$
\WHILE{$m_{\tau} = \text{NONE} \land \bigcup_{a \in \setAgents} S_a \neq \emptyset$}
    \STATE $S_t \leftarrow \left\{ (p, a) \in \setPlayer \times \setArms \mid a \in S_p(t) \lor p \in S_a(t) \right\}$
    \FOR{$m \in \mathfrak{M}(S_t)$} 
        \STATE Sample update $\hat{\mu}_{a,m(a)} \; \forall a \in \setAgents$ 
    \ENDFOR
    \STATE Update $c_{a,i}(t) \; \forall a\in \setAgents ,\; i \in OS(a)$ \hfill $\triangleright$ Eq. \ref{eq:ci}
    \STATE Update $P_a \; \forall a \in \setAgents$ \hfill $\triangleright$ Eq. \ref{eq:update_pref}
    \STATE Eliminate pairs in $S_a$ \hfill $\triangleright$ Eq. \ref{eq:elim}
    \IF{$\matching = SSM(\{P_a\}_{a \in \setAgents})$ exists} 
        \FORALL{$(p,a) \notin \matching$ with $p\in S_a$ or $a\in S_p$}
            \STATE \textbf{if} $\matching(p) \succ_{P_p} a$ \textbf{then} $S_p \leftarrow S_p \setminus \{a\}$
            \STATE \textbf{if} $p \succ_{P_a} \matching(a)$ \textbf{then} $S_a \leftarrow S_a \setminus \{p\}$
        \ENDFOR
    \ENDIF
    \STATE $m_{\tau}= POSM(\{P_a\}_{a \in \setAgents})$
\ENDWHILE
\RETURN $m_{\tau}$
\end{algorithmic}
\end{algorithm}

Now we provide a proof of our Proposition~\ref{prop: a-valid-profile} that we restate bellow:
\begin{proposition}[Restated]\label{prop: a-valid-profile-appendix}
    Let $F^*$ be a full ranking profile for agents $\setAgents = \setPlayer \cup \setArms$ with $\setPlayer$-optimal stable matching $\matching^*$. 
    Then the following partial profile $P$ is a valid profile for $F^*$:
\begin{itemize}
    \item $\begin{aligned}[t]
        \forall p\in\setPlayer,\quad  P_p &= \{(a, a') \in \setArms\times\setArms \mid a \succ a' \succeq_{F^*_p} \matching^*(p)\} \\
                 &\cup \{(a, a') \in \setArms\times\setArms \mid a \succeq_{F^*_p} \matching^*(p) \succ_{F^*_p} a'\}
          \end{aligned}$
    \item $\begin{aligned}[t]
        \forall a\in\setArms, \quad P_a &= \{(p, p') \in \setArms\times\setArms \mid p' \prec p \preceq_{F^*_a} \matching^*(a)\} \\
                 &\cup \{(p, p') \in \setArms\times\setArms \mid p \preceq_{F^*_a} \matching^*(a) \prec_{F^*_a} p'\}
          \end{aligned}$
\end{itemize}
\end{proposition}
\begin{proof}
    Consider some full ranking profile $F$ compatible with $P$. 
    We first show that $\matching^*$ is a stable matching in $F$, and then show that there exists no rotation in $F$ that improves $\setPlayer$ partners, i.e., $\matching^*$ is optimal in $F$.

    For contradiction, assume there exists a pair $(p,a)$ blocking $\matching^*$ in $F$, i.e., $a \succ_{F_p} \matching^*(p)$ and $p \succ_{F_a} \matching^*(a)$.
    In $P$ all relations of $\matching^*(i)$ to all other agents on the other side of $i$ are known in $P_i$, and coincide with the relations in $F^*$. Because $F$ is compatible with $P$, we have $a \succ_{F^*_p} \matching^*(p)$ and $p \succ_{F^*_a} \matching^*(a)$. But then $(p,a)$ blocks $\matching^*$ in $F^*$ --- a contradiction to $\matching^*$ being stable in $F^*$.

    Now assume, for contradiction, that there exists a rotation $(p_1, a_1), (p_2, a_2), \dots, (p_\ell, a_\ell) \in \matching^*$ whose elimination,  i.e., matching any $p_i$ with $a_{i+1}$ instead of $a_i = \matching^*(p_i)$ while maintaining all other matches in $\matching^*$, 
    \begin{itemize}
        \item improves all $p_i$'s preferences, i.e.: \\
        $a_{(i+1) \mod N} \succ_{F_{p_i}} a_i = \matching^*(p_{i})$ for all $i=1,\dots,\ell$
        \item deteriorates all $a_i$'s preferences, i.e.: \\
        $p_{(i-1) \mod N} \prec_{F_{a_i}} p_i = \matching^*(a_{i})$ for all $i=1,\dots,\ell$.
    \end{itemize}
    Again, all these relations are present in $P$ and because of this also hold in $F^*$. Thus, $(p_1, a_1), (p_2, a_2), \dots, (p_\ell, a_\ell) \in \matching^*$ is a rotation in $F^*$ that improves $\setPlayer$-preferences in $F^*$ --- a contradiction to the optimality of $\matching^*$. 

    Because $F$ was chosen arbitrarily, $\matching^*$ is the optimal stable matching in any full ranking profile compatible with $P$. This proves that $P$ is a valid profile for $F^*$.  
\end{proof} 

\section{Lower bound}
\label{app:lower_bound1}
We consider a matching problem consisting of two disjoint sets of agents, $\setPlayer$ and $\setArms$, with $K=|\setPlayer|$ and $N=|\setArms|$ number of agents in each set. Let $\setAgents = \setPlayer \cup \setArms$ denote the set of all agents. For any agent $a \in \setAgents$, we define $OS(a)$ as the set of agents on the opposite side of the market i.e., \( OS(\agent) = \setArms\) if \( \agent \in \setPlayer \), and \( OS(\agent) = \setPlayer \) if \( \agent \in \setArms \).

Let $\mathcal{I}$ be a class of matching environments, where we define a matching instance $\nu = \{ \nu_{a,i} \}_{a \in \mathfrak{A}, i \in OS(a)} \in \mathcal{I}$ as a collection of probability distributions, where each $\nu_{a,i}$ is associated with the agent-partner pair $(a, i)$. For every pair, the observed reward $X_{a,i}$ is a random variable drawn from $\nu_{a,i}$ with a true underlying mean reward $\mu_{a,i} = \mathbb{E}_{X \sim \nu_{a,i}}[X]$, representing the utility agent $a$ derives from being matched with $i$.

For an instance $\nu \in \mathcal{I}$, we denote with $\matching^{\star}_{\nu}$ the optimal stable matching. To provide a lower bounds on the sample complexity, we define the set of alternative environments $\mathcal{I}_{\text{alt}}(\nu)$ as the collection of all environments in $\mathcal{I}$ where the optimal stable matching is different from that in $\nu$:
\begin{equation}
\mathcal{I}_{\text{alt}}(\nu) = \{ \nu' \in \mathcal{I} : \matching^{\star}_{\nu'} \neq \matching^{\star}_{\nu} \}
\end{equation}
We denote with $P_\nu(\mathcal{E})$ the probability of an event $\mathcal{E} \in \mathcal{F}_t$ under the environment $\nu$.

\subsection{Proof of Lower bound}

\begin{theorem}[Restated]
Fix $\delta \in (0,1/2)$. For any $\delta$-PACOS algorithm and any matching instance $\nu$,
\begin{equation}
    \mathbb{E}_{\nu}[\tau] \ge \ln{\frac{1}{2.4\delta}} c_{\star}^{-1}(\nu)
\end{equation}
where $c_{\star}(\nu)$ is the solution of the following optimization problem:
\begin{equation}
    c_{\star}(\nu) = \sup_{w\in \mathcal{S}_{\mid \mathcal{M}\mid}} \inf_{\nu' \in \mathcal{I}_{\text{alt}}(\nu)} \sum_{a \in \mathfrak{A}} \sum_{i \in OS(a)} w_{a,i} KL(\nu_{a,i}, \nu'_{a,i} )
\end{equation}
\end{theorem}
\begin{proof}
Let $\delta \in (0,1/2)$ and consider a $\delta-$PACOS algorithm, where the probability of the event that the algorithm outputs the optimal stable matching is at least $1-\delta$, i.e., the probability of the event $\mathcal{E}_{\nu} = \{m_{\tau} = m^{\star}_{\nu} \}$ satisfies $P_{\nu}(\mathcal{E}_{\nu})\geq 1 - \delta $. Now consider alternative matching instance  $\nu' \in \mathcal{E}_{\text{alt}}(\nu)$, where by definition we have that $P_{\nu'}(\mathcal{E}_{\nu})\leq \delta$.

Applying the Lemma~\ref{lemma_dist} to the event $\mathcal{E}_{\nu}$ we have that $\forall \nu' \in \mathcal{I}_{\text{alt}}(\nu)$:
 \begin{equation}
    \sum_{a \in \mathfrak{A}} \sum_{i \in OS(a)} \mathbb{E}_\nu[N_{a,i}(\tau)] KL(\nu_{a,i}, \nu'_{a,i}) \ge d(P_{\nu}(\mathcal{E}_{\nu}),P_{\nu'}(\mathcal{E}_{\nu})) \ge d(1-\delta,\delta)
 \end{equation}
where $d(x,y) := x \log(x/y) + (1-x) \log((1-x)/(1-y))$ is the binary relative entropy.

Now we can further bound:
 \begin{equation}
    \sum_{a \in \mathfrak{A}} \sum_{i \in OS(a)} \mathbb{E}_\nu[N_{a,i}(\tau)] KL(\nu_{a,i}, \nu'_{a,i})  \ge \ln{\frac{1}{2.4 \delta}}
 \end{equation}
 
using the lower bound on the binary relative entropy $d(1-\delta, \delta) \ge \ln \frac{1}{2.4 \delta}$ for $\delta \in (0, 1/2)$ \cite{kaufmann2016complexity}.

Taking the infimum over the possible alternative model $\nu'  \in \mathcal{I}_{\text{alt}}(\nu)$ we have that:
 \begin{align}
    \ln{\frac{1}{2.4 \delta}} &\leq \inf_{\nu'  \in \mathcal{E}_{\text{alt}}(\nu)}
    \sum_{a \in \mathfrak{A}} \sum_{i \in OS(a)} \mathbb{E}_\nu[N_{a,i}(\tau)] KL(\nu_{a,i}, \nu'_{a,i}) \\
    &\leq \mathbb{E}_\nu[\tau] \inf_{\nu'  \in \mathcal{I}_{\text{alt}}(\nu)}
    \sum_{a \in \mathfrak{A}} \sum_{i \in OS(a)} \frac{\mathbb{E}_\nu[N_{a,i}(\tau)] }{\mathbb{E}_\nu[\tau]}KL(\nu_{a,i}, \nu'_{a,i})
 \end{align}

Now by the definition of the transformed simplex we have that $\frac{\mathbb{E}_\nu[N_{a,i}(\tau)] }{\mathbb{E}_\nu[\tau]} \in \mathcal{S}_{\mid \mathcal{M}\mid}$ which leeds to:
 \begin{equation}
    \ln{\frac{1}{2.4 \delta}} \leq \mathbb{E}_\nu[\tau] \sup_{w \in \mathcal{S}_{\mid \mathcal{M}\mid}} \inf_{\nu'  \in \mathcal{I}_{\text{alt}}(\nu)}
    \sum_{a \in \mathfrak{A}} \sum_{i \in OS(a)} 
    w_{a,i} KL(\nu_{a,i}, \nu'_{a,i})
 \end{equation}
which concludes the proof
\end{proof}

\subsection{Change of Distribution Lemma}
At each time step $t \geq 1 $ we sample a matching $m_t \in \mathcal{M}$, and the platform observers the rewards of each agent that is matched $X_t = (X_{a,m_t(a)} I\{a \in m_t(a)\})_{\forall a \in \setAgents} \subset \mathbb{R}^{K \times N}$ similar to the semi-bandit feedback, where $X_{a,m_t(a)}$ is distributed according to $\nu_{a,m_t(a)}$.

Given a history $\mathcal{F}_{t} = \sigma(m_1, X_{1,m_1},  \cdots, m_t, X_{t,m_t} )$ up to time $t$, we can define the log-likelihood of the observation for the two different matching instances $\nu$ and $\nu'$:
\begin{equation}
    \mathcal{L}_t(m_1, X_{1,m_1},  \cdots, m_t, X_{t,m_t} ) = \sum_{a\in \setAgents} \sum_{i \in OS(a)} \sum_{s=1}^{t} I\{ (a,i) \in m_t\} \ln \frac{p_{a,i}(X_{s,a,i})}{p'_{a,i}(X_{s,a,i})}
\end{equation}
where $p_{i,j}$ denotes the density of the distribution $\nu_{i,j}$.

Now, we present a transportation lemma for two-sided matching problems, adapted from Lemma 1 originally proposed in \cite{kaufmann2016complexity}.

\begin{lemma}[Change of Distribution for Matching Markets]
\label{lemma_dist}
Let $\nu$ and $\nu'$ be two matching market instances where for every agent $a \in \mathfrak{A}$ and every potential partner $i \in OS(a)$, the reward distributions $\nu_{a,i}$ and $\nu'_{a,i}$ are mutually absolutely continuous. For any almost-surely finite stopping time $\tau$ with respect to $(\mathcal{F}_t)_{t \in \mathbb{N}}$, and any event $\mathcal{E} \in \mathcal{F}_\sigma$:
\begin{equation*}
    \sum_{a \in \mathfrak{A}} \sum_{i \in OS(a)} \mathbb{E}_\nu[N_{a,i}(\tau)] KL(\nu_{a,i}, \nu'_{a,i}) \ge d(P_\nu(\mathcal{E}), P_{\nu'}(\mathcal{E}))
\end{equation*}
where $d(x,y) := x \log(x/y) + (1-x) \log((1-x)/(1-y))$ is the binary relative entropy.
\end{lemma}

\begin{proof}
For each $(a,i)\in \setAgents \times OS(a)$, let $\{X_{s,a,i}\}_{s\ge 1}$ be the observation sequence associated with that pair, and let $N_{a,i}(\tau)$ be the number of times it is sampled up to time $\tau$; thus the realized observations by time $\tau$ are $\{X_{s,a,i}\}_{s=1}^{N_{a,i}(\tau)}$.

Then we can define the the log-likelihood $\mathcal{L}_\tau$ at time $\tau$ as
\begin{equation}
    \mathcal{L}_\tau(m_1, X_{1,m_1},  \cdots, m_\tau, X_{\tau,m_\tau} ) = \sum_{a\in \setAgents} \sum_{i \in OS(a)} \sum_{s=1}^{N_{a,i}(\tau)} \ln \frac{p_{a,i}(X_{s,a,i})}{p'_{a,i}(X_{s,a,i})}.
\end{equation}

With a direct application of Wald's Identity on $\mathbb{E}_{\nu}[L_{\tau}]$ we have that:
\begin{equation}
\mathbb{E}_{\nu}[L_{\tau}] = \sum_{a\in \setAgents} \sum_{i \in OS(a)} \mathbb{E}_{\nu}[N_{a,i}(\tau)] \mathbb{E}_{\nu}[\ln\frac{p_{a,i}(X_{a,i})}{p'_{a,i}(X_{a,i})}] = \sum_{a\in \setAgents} \sum_{i \in OS(a)} \mathbb{E}_{\nu}[N_{a,i}(\tau)] KL(\nu_{a,i},\nu'_{a,i})]
\end{equation}
where in the second equality we use the fact that $KL(\nu, \nu') = \mathbb{E}_{\nu}[\ln\frac{p(X)}{p'(X)}]$.

Applying the Lemma 19 from \cite{kaufmann2016complexity} which state that $\mathbb{E}_{\nu}[L_{\tau}]  \ge d(P_\nu(\mathcal{E}), P_{\nu'}(\mathcal{E}))$ results in:

\begin{equation*}
    \sum_{a \in \mathfrak{A}} \sum_{i \in OS(a)} \mathbb{E}_\nu[N_{a,i}(\tau)] KL(\nu_{a,i}, \nu'_{a,i}) \ge d(P_\nu(\mathcal{E}), P_{\nu'}(\mathcal{E}))
\end{equation*}

which concludes the proof.
\end{proof}

\section{Lower bounds on specific instances}
\label{app:lower_bound2}
We can now continue by providing a lower bound based on specific instances.

Similar to the work of \cite{kaufmann2016complexity} we will perform an analysis changing the distribution of one pair of agents, under the identifiable class of environments:
\begin{equation} 
\mathcal{I}_{\text{class}} = \{ \nu = (\nu_{p,a})_{\forall (p,a) \in \mathcal{P} \times \mathcal{A}} : \nu_{p,a} \in \mathcal{P}_m \text{ and }  m^{\star}(\nu) \ne \emptyset \}
\end{equation}

such that the probability measure $\mathcal{P}_m$ satisfies the following continuity assumption:
\begin{assumption}
\label{assuption:indent}
For any two distribution $p,q \in \mathcal{P}_m$, such that $p \ne q$, and for all $\alpha>0$:
\begin{enumerate}
    \item There exist a $q_1 \in \mathcal{P}_m$ such that: $KL(p,q) < KL(p, q_1) < KL(p,q)  +\alpha$ and $\mathbb{E}_{X \sim q_1}[X] > \mathbb{E}_{X \sim q}[X] $
    \item There exist a $q_2 \in \mathcal{P}_m$ such that: $KL(p,q) < KL(p, q_2) < KL(p,q)  +\alpha$ and $\mathbb{E}_{X \sim q_2}[X] < \mathbb{E}_{X \sim q}[X] $
\end{enumerate}
\end{assumption}

\subsection{Global preferences }
We study instances with global preferences similar to Example 7 in \cite{pmlr-v108-liu20c}.

\begin{lemma}[Global preferences]
\label{lemma:global_pref}
Consider a matching instance $\nu \in \mathcal{I}_{\text{class}}$ with equal umber of agents on both sides and that satisfies Assumption \ref{assuption:indent}. Suppose the agents exhibit global preferences such that for all $p_i \in \mathcal{P}$, $a_1 \succ a_2 \succ \dots \succ a_n$, and for all $a_i \in \mathcal{A}$, $p_1 \succ p_2 \succ \dots \succ p_n$. Then any $\delta$-PACOS Algorithm, satisfies that:
\begin{align}
\mathbb{E}_{\nu}[\tau] \geq \log\left(\frac{1}{2.4 \delta}\right)  \max(\max_{p_i \in \setPlayer}A_i, \max_{a_i \in \setArms}B_i)
\end{align}
where
\begin{align*}
    A_i = \max\Bigr( \frac{1}{KL(\nu_{p_i,a_i}, \nu_{p_i,a_{i+1}})}, \frac{1}{KL(\nu_{a_i,p_i}, \nu_{a_i,p_{i+1}}) }  \Bigl) + \sum_{j<i} \frac{1}{KL(\nu_{a_j,p_i}, \nu_{a_j,p_j})} + \sum_{j>i} \frac{1}{KL(\nu_{p_i,a_j}, \nu_{p_i,a_i})}
\end{align*}
and 
\begin{align*}
    B_i = \max\Bigr( \frac{1}{KL(\nu_{p_i,a_i}, \nu_{p_i,a_{i+1}})}, \frac{1}{KL(\nu_{a_i,p_i}, \nu_{a_i,p_{i+1}}) } \Bigl) + \sum_{j<i} \frac{1}{KL(\nu_{p_j,a_i}, \nu_{p_j ,a_j}) } 
+ \sum_{j>i} \frac{1}{KL(\nu_{a_i,p_j}, \nu_{a_i,p_i})}
\end{align*}

\end{lemma}

\begin{proof}
We consider an instance $\nu$ with equal umber of agents on both sides and global preference profile $F$ where for all $p_i \in \mathcal{P}$, $a_1 \succ a_2 \succ \dots \succ a_n$, and for all $a_i \in \mathcal{A}$, $p_1 \succ p_2 \succ \dots \succ p_n$, such that the optimal stable matching $\matching$ is given by $\matching(p_i)=a_i$ for all $i$.

Figure~\ref{fig:lower-bound-graph-exp1} illustrates the case of three agents on each side.

\begin{figure}[h!]
    \centering
    \begin{tikzpicture}[
    leftnode/.style={circle, draw, fill=blue!10, minimum size=7mm},
    rightnode/.style={circle, draw, fill=red!10, minimum size=7mm},
    pref/.style={font=\scriptsize, align=right},
    prefR/.style={font=\scriptsize, anchor=west, align=left},
    edge/.style={gray!60},
    boldedge/.style={black, very thick}
]

\foreach \i in {1,...,3} {
    \node[leftnode] (P\i) at (0,6-\i) {$p_{\i}$};
    \node[pref, anchor=east] at (-0.4,6-\i)
        {$a_1 > a_2 > a_3$};
}

\node[rightnode] (A1) at (6,5) {$a_1$};
\node[prefR] at (6.4,5)
{$p_1 > p_2 > p_3$};

\node[rightnode] (A2) at (6,4) {$a_2$};
\node[prefR] at (6.4,4)
{$p_1 > p_2 > p_3$};

\node[rightnode] (A3) at (6,3) {$a_3$};
\node[prefR] at (6.4,3)
{$p_1 > p_2 > p_3$};

\foreach \i in {1,...,3} {
    \foreach \j in {1,...,3} {
        \draw[edge] (P\i) -- (A\j);
    }
}

\draw[boldedge] (P1) -- (A1);
\draw[boldedge] (P2) -- (A2);
\draw[boldedge] (P3) -- (A3);

\end{tikzpicture}
    \caption{Instance with global preferences for $N=3$.}
    \label{fig:lower-bound-graph-exp1}
\end{figure}
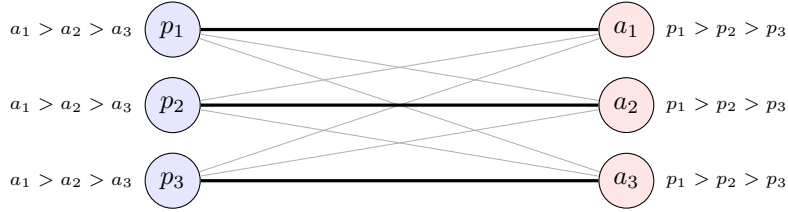

We first present a valid preference profile $P$ for this instance that is minimal, in the sense that $P$ is a subset of every other valid partial profile $P'$ for $F$, i.e., $\forall P' \in \mathcal{P}_{\mathrm{valid}}(F)$, it holds that $P_a \subseteq P'_a$ for all $a \in \setAgents$.
This profile $P$ is defined as follows: Each agent $p_i \in \setPlayer$ has partial preferences $P_{p_i}= \{a_i \succ a_j \mid j>i\}$, i.e., $p_i$ prefers their optimal matching partner $a_i$ over all higher-indexed agents in $\setArms$. Similarly, for each agent $a_i \in \setArms$, $P_{a_i} = \{p_i \succ p_j \mid  j>i\}$.

\paragraph{Validity of $P$} \textbf{} \\

By definition, $P$ is compliant with $F$.
To prove that $P \in \mathcal{P}_{\mathrm{valid}}(F)$, we first show that $\matching$ is a super stable matching in $P$, and then show that $\matching$ is optimal in every completion of $P$, i.e., any full ranking profile compatible with $P$.

To verify that $\matching(p_i)=a_i$ is super stable under the partial preference profile $P$, consider any pair $(p_i,a_j)$ with $j\neq i$.
If $i>j$, then $a_j$ prefers their match $p_j$ to $p_i$ in $P$, since $p_j = \matching(a_j)\succ_{P_{a_j}} p_i$. Thus, $(p_i,a_j)$ is not a blocking pair in any completion of $P$. 
If $j>i$, then $p_i$ prefers their match $a_i$ to $a_j$ in $P$, since $a_i = \matching(p_i)\succ_{P_{p_i}} a_j$, and again $(p_i,a_j)$ is not blocking $\matching$ in any completion of $P$. Hence, no pair blocks $\matching$ in any completion of $P$, and $\matching$ is super stable.

We now show that $\matching$ is also the optimal stable matching in any completion of $P$, i.e., a pervasive stable matching. 
For contradiction, assume that this is not the case, and that there exists a full ranking profile $F'$ compatible with $P$ that admits an optimal stable matching $\matching' \neq \matching$. Because $\matching$ is super stable, it is also a stable matching in $F'$. Now consider some  edge $(p_i,a_j)$ in $\matching'\setminus\matching$. Because $\matching'$ is assumed to be optimal in $F'$, we have that $p_i$ is better off under $\matching'$ than $\matching$ and vice versa for $a_j$, i.e.,   $a_j = \matching'(p_i) \succ_{F'_{p_i}} \matching(p_i) = a_i$ and $p_j = \matching(a_j) \prec_{F'_{a_j}} \matching'(a_j) = p_i$. 
If $i > j$ this is not compliant with preferences $p_j \succ_{P_{a_j}} p_i$ in $P$.
Similarly, if $j>i$ this is not compliant with preferences $a_i \succ_{P_{p_j}} a_j$ in $P$.
Thus, there cannot exist any  edge $(p_i,a_j)$ in $\matching'\setminus\matching$. This is a contradiction to our choice of $F'$ and $\matching'$. Hence $\matching$ is optimal in any full ranking profile compatible with $P$.

\paragraph{Minimality of $P$} \textbf{} \\
We now prove that $P$ is minimal in the sense that it is a subset of every other valid partial profile $P'$ for $F$, i.e., $\forall P' \in \mathcal{P}_{\mathrm{valid}}(F)$, it holds that $P_a \subseteq P'_a$ for all $a\in\setAgents$.

To see this, consider that for any valid profile $P' \in \mathcal{P}_{\mathrm{valid}}(F)$, in order for an edge $(p_i, a_j)$ not to block the stable matching $\matching$ in any completion of $P'$, we require that either $a_i \succ_{P'_{p_i}} a_j$ or $p_j \succ_{P'_{a_j}} p_i$ (not necessarily both). According to $F$, the former holds only when $i < j$, and the latter holds only when $j < i$. Thus:
\begin{itemize}
\item For all pairs $(p_i, a_j)$ with $i < j$, a valid profile must specify $a_i \succ_{p_i} a_j$.
\item For all pairs $(p_i, a_j)$ with $i > j$, a valid profile must specify $p_j \succ_{a_j} p_i$.
\end{itemize}

These are exactly the relations specified in $P$.
Therefore, $P$ is both valid and minimal with respect to ensuring stability and optimality of $\matching$ across all completions.

\paragraph{Lower bound on the expected number of pair samples $\mathbb{E}_\nu[N_{p_i,a_j}(\tau)]  $} \textbf{} \\
Following \cite{kaufmann2016complexity}, we construct alternative environments to provide lower bounds on the sample counts $N_{p_i,a_i}$ and $N_{a_i,p_i}$ for any agent pair $(p_i,a_i)$.

We construct our alternative environments based on the preference profile $P = \{P_a\}_{a \in \mathfrak{A}}$ defined above. Specifically, since we proved that $P$ is valid and minimal, any completion (or alternative environment) compatible with $P$ will yield the exact same optimal stable matching. Therefore, to create an alternative environment $\nu'$ where the optimal stable matching is different, we must construct an environment that is not consistent with $P$. We achieve this by minimally modifying the reward distribution of only a single pair. 

We specifically consider  alternative instances $\nu'$ where the distribution of one  agent $p_i \in \setPlayer$ is modified in only one agent $a_j \in \setArms$:
\begin{align}
       \nu' = (&\nu_{p_1}, \cdots, \nu'_{p_i}, \cdots, \nu_{p_K}  \\ 
       &\nu_{a_1}, \cdots,  \nu_{a_i},  \cdots, \nu_{a_N} ) 
\end{align}
where
\begin{equation}
   \nu'_{p_i} = (\nu_{{p_{i},a_{1}}}, \cdots,  \nu'_{p_{i},a_{j}},  \cdots, \nu_{p_{i},a_{N}} ) 
\end{equation}
with (using Assumption~\ref{assuption:indent}):
\begin{itemize}
    \item if $j=i$ then: $\mu'_{p_{i}, a_{i}} < \mu_{p_{i}, a_{i+1}} $  and $KL(\nu_{p_{i}, a_{i}},\nu_{p_{i}, a_{i+1}}) < KL(\nu_{p_{i}, a_{i}},\nu'_{p_{i}, a_{i}}) < KL(\nu_{p_{i}, a_{i}},\nu_{p_{i}, a_{i+1}}) +\alpha$
    \item if $j>i$ then: $\mu'_{p_{i}, a_{j}} > \mu_{p_{i}, a_{i}}$  and $KL(\nu_{p_{i}, a_{j}}, \nu_{p_{i}, a_{i}}) < 
                  KL(\nu_{p_{i}, a_{j}}, \nu'_{p_{i}, a_{j}}) < KL(\nu_{p_{i}, a_{j}}, \nu_{p_{i}, a_{i}}) + \alpha $
\end{itemize}

Note that in every of those alternative instances we have a different optimal stable matching, i.e., $m^{\star}_{\nu} \neq m^{\star}_{\nu'}$. Thus, introducing the event $\mathcal{E} = \{m_{\tau} = m^{\star}_{\nu}\}$, any $\delta$-PAC algorithm stopping at time $\tau$ with matching $m_\tau$ satisfies $\Pr_{\nu}(\mathcal{E})\geq 1 - \delta$ and $\Pr_{\nu'}(\mathcal{E}) \leq \delta$ .

Applying Lemma~\ref{lemma_dist} on the stopping time $\tau$ for the two matching instances $\nu$ and $\nu'$ we have that:
\begin{equation*}
    \sum_{a \in \mathfrak{A}} \sum_{i \in OS(a)} \mathbb{E}_\nu[N_{a,i}(\tau)] KL(\nu_{a,i}, \nu'_{a,i}) \ge d(P_\nu(\mathcal{E}), P_{\nu'}(\mathcal{E})) \ge d(1-\delta, \delta)
\end{equation*}
where $N_{a,i}(\tau)$ is the number of samples of pair $(a,i)\in \setPlayer\times\setArms$ until stopping time $\tau$.

Using the lower bound on the binary relative entropy $d(1-\delta, \delta) \ge \ln \frac{1}{2.4 \delta}$ for $\delta \in (0, 1/2)$ \cite{kaufmann2016complexity}, and the fact that the two different instances $\nu$ and $\nu'$ differ only in the distribution of one pair, we obtain:
\begin{equation*}
\mathbb{E}_\nu[N_{p_i,a_j}(\tau)]  \ge \log(\frac{1}{2.4 \delta}) \frac{1}{KL(\nu_{p_i,a_j}, \nu'_{p_i,a_j})}
\end{equation*}
Now from the definition of the alternative instances, for $i=j$:
\begin{equation*}
\mathbb{E}_\nu[N_{p_i,a_i}(\tau)]  \ge \log(\frac{1}{2.4 \delta}) \frac{1}{KL(\nu_{p_i,a_i}, \nu_{p_i,a_{i+1}}) + \alpha}
\end{equation*}
and for $j>i$ :
\begin{equation*}
\mathbb{E}_\nu[N_{p_i,a_j}(\tau)]  \ge \log(\frac{1}{2.4 \delta}) \frac{1}{KL(\nu_{p_i,a_j}, \nu_{p_i,a_i})+ \alpha}
\end{equation*}
Furthermore, for $j < i$, the definition of $P$ and the construction of the alternative environments imply that $p_i$ does not require samples from $a_j$. Consequently:
\begin{equation}
\label{eq:case_3_p}
\mathbb{E}_\nu[N_{p_i,a_j}(\tau)] = 0
\end{equation}

Now following the same reasoning for the agents $a_i \in \setArms$, by symmetry of the partial preferences $P$ we have that for $i=j$:
\begin{equation*}
\mathbb{E}_\nu[N_{a_i,p_i}(\tau)]  \ge \log(\frac{1}{2.4 \delta}) \frac{1}{KL(\nu_{a_i,p_i}, \nu_{a_i,p_{i+1}}) + \alpha}
\end{equation*}
and for $j>i$ :
\begin{equation*}
\mathbb{E}_\nu[N_{a_i,p_j}(\tau)]  \ge \log(\frac{1}{2.4 \delta}) \frac{1}{KL(\nu_{a_i,p_j}, \nu_{a_i,p_i})+ \alpha}
\end{equation*}
and for $j<i$ :
\begin{equation}
\label{eq:case_3_a}
\mathbb{E}_\nu[N_{a_i,p_j}(\tau)] = 0
\end{equation}

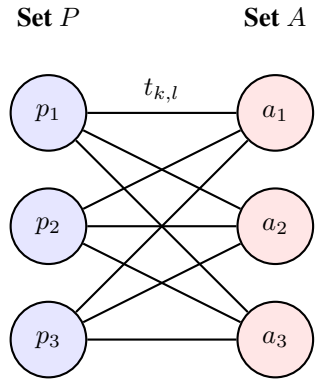
\begin{figure}[h]
    \centering
    \begin{tikzpicture}[
        thick,
        agent/.style={circle, draw, minimum size=1cm, inner sep=0pt, fill=blue!10},
        arm/.style={circle, draw, minimum size=1cm, inner sep=0pt, fill=red!10},
        connection/.style={-, thick}
    ]

        \node[agent] (p1) at (0, 3) {$p_1$};
        \node[agent] (p2) at (0, 1.5) {$p_2$};
        \node[agent] (p3) at (0, 0) {$p_3$};
        
        \node[above=0.5cm of p1, font=\bfseries] {Set $P$};

        \node[arm] (a1) at (3, 3) {$a_1$};
        \node[arm] (a2) at (3, 1.5) {$a_2$};
        \node[arm] (a3) at (3, 0) {$a_3$};

        \node[above=0.5cm of a1, font=\bfseries] {Set $A$};

        
        \draw[connection] (p1) -- node[midway, above] {$t_{k,l}$} (a1);%
        \draw[connection] (p1) -- (a2);
        \draw[connection] (p1) -- (a3);
        \draw[connection] (p2) -- (a1);
        \draw[connection] (p2) -- (a2);
        \draw[connection] (p2) -- (a3);
        \draw[connection] (p3) -- (a1);
        \draw[connection] (p3) -- (a2);
        \draw[connection] (p3) -- (a3);

    \end{tikzpicture}
    \caption{Illustration of the bipartite graph $G$.}
    \label{fig:bipartite_graph2}
\end{figure}

\textbf{Lower bound on the number of matchings $\mathbb{E}_{\nu}[\tau]$} \textbf{} \\
Finally, we lower bound the number of matchings $\tau$ by analyzing the associated bipartite multigraph $G$ with edge multiplicities $t_{i,j}$, as illustrated in Figure~\ref{fig:bipartite_graph2}.

To define the edge multiplicities $t_{k,l}$, we observe that each pair $(p_k, a_l)$ must be sampled at least $\max(\mathbb{E}_{\nu}[N_{p_k,a_l}(\tau)], \mathbb{E}_{\nu}[N_{a_k,p_l}(\tau)])$. Letting parameter $\alpha \to 0$, we have that for $k=l$: 
\begin{equation*}
    t_{k,k} =  \max( \mathbb{E}_\nu[N_{p_k,a_k}(\tau)], \mathbb{E}_\nu[N_{a_k,p_k}(\tau)]) \geq \log(\frac{1}{2.4 \delta}) \max\Bigr( \frac{1}{KL(\nu_{p_k,a_k}, \nu_{p_k,a_{k+1}})}, \frac{1}{KL(\nu_{a_k,p_k}, \nu_{a_k,p_{k+1}}) } \Bigl)
\end{equation*}
and for $k<l$, Eq.~\ref{eq:case_3_a} implies that $\mathbb{E}_\nu[N_{a_l,p_k}] = 0$ and thus:
\begin{equation*}
    t_{k,l} =  \mathbb{E}_\nu[N_{p_k,a_l}(\tau)] \geq \log(\frac{1}{2.4 \delta}) \frac{1}{KL(\nu_{p_k,a_l}, \nu_{p_k,a_k})} 
\end{equation*}
and for $l<k$, Eq.~\ref{eq:case_3_p} implies that $\mathbb{E}_\nu[N_{p_k,a_l}] = 0$ and thus:
\begin{equation*}
    t_{k,l} =  \mathbb{E}_\nu[N_{a_l,p_k}(\tau)] \geq \frac{1}{KL(\nu_{a_l,p_k}, \nu_{a_l,p_l})}
\end{equation*}
 
Now according to Kőnig’s Theorem, the number of matchings required to cover the edges of the bipartite multigraph $G$, is equal to the maximum degree: 
$$\mathbb{E}_{\nu}[\tau] = def(G) = \max(\max_{p_i \in \setPlayer} deg(p_i),\max_{a_i \in \setArms} deg(a_i))$$
where the degree for agents $p_i \in \setPlayer$ is:
\begin{align*}
    &deg(p_i) = \sum^{n}_{j=1} t_{i,j} =  t_{i,i} + \sum_{j<i} t_{i,j} + \sum_{j>i} t_{i,j} \\
    &\geq \log\left(\frac{1}{2.4 \delta}\right) \left[ \max\left( \frac{1}{KL(\nu_{p_i,a_i}, \nu_{p_i,a_{i+1}})}, \frac{1}{KL(\nu_{a_i,p_i}, \nu_{a_i,p_{i+1}}) } \right) + \sum_{j<i} \frac{1}{KL(\nu_{a_j,p_i}, \nu_{a_j,p_j})} + \sum_{j>i} \frac{1}{KL(\nu_{p_i,a_j}, \nu_{p_i,a_i})} \right]
\end{align*}
and the degree for agents for $a_i \in \setArms$:
\begin{align*}
    &deg(a_i) = \sum^{n}_{j=1} t_{j,i} =  t_{i,i} + \sum_{j<i} t_{j,i} + \sum_{j>i} t_{j,i} \\
    &\geq \log\left(\frac{1}{2.4 \delta}\right) 
\left[ \max\Bigr( \frac{1}{KL(\nu_{p_i,a_i}, \nu_{p_i,a_{i+1}})}, \frac{1}{KL(\nu_{a_i,p_i}, \nu_{a_i,p_{i+1}}) } \Bigl) + \sum_{j<i} \frac{1}{KL(\nu_{p_j,a_i}, \nu_{p_j ,a_j}) } 
+ \sum_{j>i} \frac{1}{KL(\nu_{a_i,p_j}, \nu_{a_i,p_i})}
\right]
\end{align*}

To conclude we have that:
\begin{align}
\mathbb{E}_{\nu}[\tau] \geq \log\left(\frac{1}{2.4 \delta}\right)  \max(\max_{p_i \in \setPlayer}A_i, \max_{a_i \in \setArms}B_i)
\end{align}
where,
\begin{align*}
    A_i = \max\Bigr( \frac{1}{KL(\nu_{p_i,a_i}, \nu_{p_i,a_{i+1}})}, \frac{1}{KL(\nu_{a_i,p_i}, \nu_{a_i,p_{i+1}}) }  \Bigl) + \sum_{j<i} \frac{1}{KL(\nu_{a_j,p_i}, \nu_{a_j,p_j})} + \sum_{j>i} \frac{1}{KL(\nu_{p_i,a_j}, \nu_{p_i,a_i})}
\end{align*}
and 
\begin{align*}
    B_i = \max\Bigr( \frac{1}{KL(\nu_{p_i,a_i}, \nu_{p_i,a_{i+1}})}, \frac{1}{KL(\nu_{a_i,p_i}, \nu_{a_i,p_{i+1}}) } \Bigl) + \sum_{j<i} \frac{1}{KL(\nu_{p_j,a_i}, \nu_{p_j ,a_j}) } 
+ \sum_{j>i} \frac{1}{KL(\nu_{a_i,p_j}, \nu_{a_i,p_i})}
\end{align*}
\end{proof}

\section{Additional Experiments}
\label{app:additional_exp}
We present additional results for the pure exploration setting in markets with different numbers of agents on the two sides. Specifically, we fix the number of arms to $K=10$ and vary the number of agents $N$ from $2$ to $50$. For each value of $N$, we run $20$ independent instances generated as in the main experiment. As shown in Figure~\ref{fig:exp3}, the observations from Section~\ref{sim:pure_exp} continue to hold across different market sizes.

\begin{figure}[t]
    \centering
    \includegraphics[width=1\linewidth]{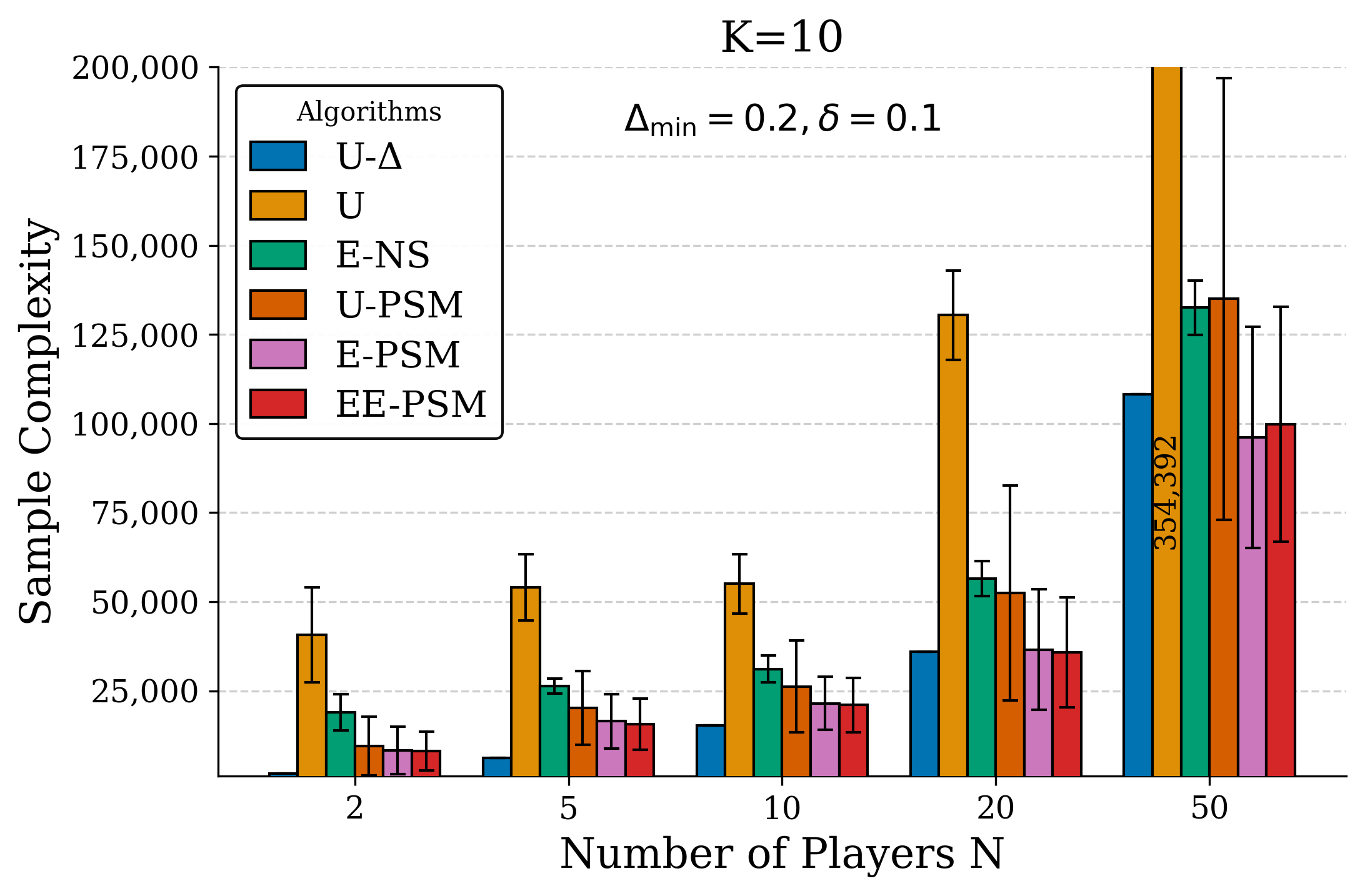}
    \caption{Sample complexity as a function of the number of players $N$, with $K=10$ fixed.}
    \label{fig:exp3}
\end{figure}

\end{document}